\documentclass{article} % For LaTeX2e
\usepackage{smile}
\usepackage{iclr2020_conference,times}
\usepackage{tabularx, booktabs,enumitem}
\usepackage{diagbox}

\usepackage{subcaption}
\usepackage{graphicx}
\usepackage{hyperref}
\usepackage{url}

\usepackage{multirow}
\usepackage{wrapfig}
\usepackage{svg}
\usepackage{tabularx, booktabs}
\newcolumntype{L}{>{\arraybackslash}X}
\newcolumntype{Y}{>{\centering\arraybackslash}X}

\usepackage{float}
\usepackage[linesnumbered,ruled,vlined]{algorithm2e}
\def\ie{{\textit{i.e.}}}
\def\eg{{\textit{e.g.}}}
\def\sics{\mbox{Scale-inv-}\mathcal{X}^2}

\usepackage{setspace}
\AtBeginDocument{%
  \addtolength\abovedisplayskip{-0.4\baselineskip}%
  \addtolength\belowdisplayskip{-0.35\baselineskip}%
  \addtolength\abovedisplayshortskip{-0.4\baselineskip}%
  \addtolength\belowdisplayshortskip{-0.35\baselineskip}%
}

\title{On the Variance of the Adaptive Learning Rate and Beyond}

\iclrfinalcopy
\author{Liyuan Liu ~\thanks{Work was done during an internship at Microsoft Dynamics 365 AI.}\\
\footnotesize{University of Illinois, Urbana-Champaign} \\
\scriptsize{\texttt{ll2@illinois}}\\
\And
Haoming Jiang~\thanks{Work was done during an internship at Microsoft Dynamics 365 AI.}\\
\footnotesize{Georgia Tech}\\
\scriptsize{\texttt{jianghm@gatech.edu}}\\
\And
Pengcheng He, Weizhu Chen\\
\footnotesize{Microsoft Dynamics 365 AI}\\
\scriptsize{\texttt{\{penhe,wzchen\}@microsoft.com}}\\
\And
Xiaodong Liu, Jianfeng Gao\\
\footnotesize{Microsoft Research}\\
\scriptsize{\texttt{\{xiaodl,jfgao\}@microsoft.com}}\\
\And
Jiawei Han\\
\footnotesize{University of Illinois, Urbana-Champaign}\\
\scriptsize{\texttt{hanj@illinois}}
}

\begin{document}

\maketitle

% !TEX encoding = UTF-8
% !TEX Root = 0_main.tex

\begin{abstract} 
The learning rate warmup heuristic achieves remarkable success in stabilizing training, accelerating convergence and improving generalization for adaptive stochastic optimization algorithms like RMSprop and Adam.
Pursuing the theory behind warmup,
we identify a problem of the adaptive learning rate -- its variance is problematically large in the early stage, and presume warmup works as a variance reduction technique.
We provide both empirical and theoretical evidence to verify our hypothesis.
We further propose Rectified Adam (RAdam), a novel variant of Adam, by introducing a term to rectify the variance of the adaptive learning rate.
Experimental results on image classification, language modeling, and neural machine translation verify our intuition and demonstrate the efficacy and robustness of RAdam.\footnote{All implementations are available at: \url{https://github.com/LiyuanLucasLiu/RAdam}.}

% Here, we study its mechanism in details. 
% This paper presents the theoretical justification of the heuristic, showing that its effectiveness is due to the fact that it reduces the problematically large variance of the adaptive learning rate in the early stage.

%Further analysis derives a reliable and concise variance approximation, and we propose Rectified Adam,
% \XD{What RAdam refers to? rectified Adam?}
%a new variant of the Adam algorithm which rectifies the variance of the adaptive learning rate.
%We conduct experiments on various datasets and observe that RAdam leads to consistent improvements over the vanilla Adam, which demonstrates that the variance issue generally exists and affects training stability and model performance~\footnote{All implementations are available at: \url{https://github.com/LiyuanLucasLiu/RAdam}}.
% \XD{Font of x-y axis is too small and hard to see. Pls make it bigger.}
% and fixing such issues leads to better robustness and efficiency. 
\end{abstract}

\section{Introduction}
\vspace{-0.1cm}

\begin{wrapfigure}{r}{0.35\textwidth} 
\centering
\vspace{-0.8cm}
\includegraphics[width=0.35\textwidth]{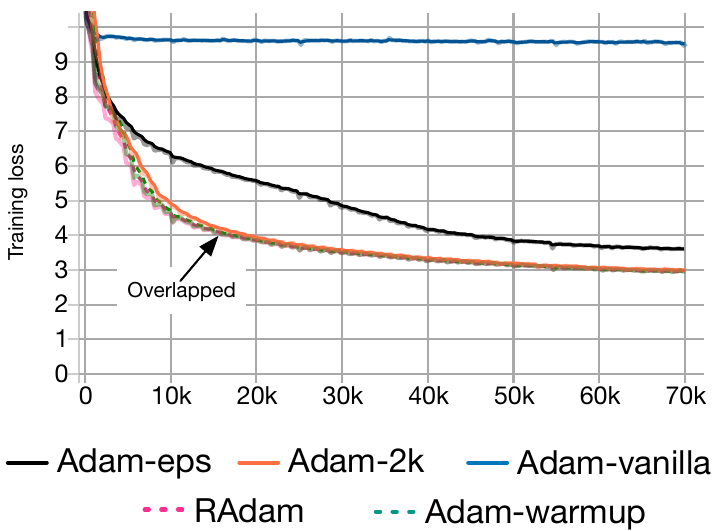}
\vspace{-0.6cm}
\caption{Training loss v.s. \# of iterations of Transformers on the De-En IWSLT'14 dataset. }
\vspace{-0.4cm}
\label{fig:eps2k}
\end{wrapfigure}

% The development of 
% fast and stable optimization algorithms 
% is one of the longest running goals in computer science~\citep{cauchy1847methode}.
Fast and stable optimization algorithms are what generations of researchers have been pursuing~\citep{gauss1823theoria,cauchy1847methode}.
Remarkably, stochastic gradient-based optimization, such as stochastic gradient descent (SGD), has witnessed tremendous success in many fields of science and engineering despite its simplicity.
%For example, stochastic gradient descent (SGD) has demonstrated its effectiveness in extensive machine learning applications. 
% Recently, many new methods have been proposed 
Recently, many efforts have been made 
to accelerate optimization by applying \textit{adaptive learning rate}. 
In particular, Adagrad~\citep{duchi2011adaptive} and its variants, \eg, RMSprop~\citep{tieleman2012lecture}, Adam~\citep{kingma2014adam}, Adadelta~\citep{zeiler2012adadelta} and Nadam~\citep{dozat2016incorporating},
% have been widely used due to their fast convergence.
stand out due to their fast convergence, and have been considered as the optimizer of choice in many applications. 

%Despite their fast convergence, it has been noticed that in many cases, 
However, it has been observed that these optimization methods may converge to bad/suspicious local optima,
% or even diverge, 
% \HMJ{Actually no experiment show that they diverges. $==>$ ..optima no matter how the hyper-parameter is chosen.}
and have to resort to a warmup heuristic -- using a small learning rate in the first few epochs of training to mitigate such problem~\citep{vaswani2017attention, popel2018training}.  
% require a warmup stage (i.e., using a small learning rate in the first few epochs); otherwise, the optimization converges to suspicious/bad local optima or even explodes~\citep{vaswani2017attention, popel2018training}.
For example, when training typical Transformers based neural machine translation models
on the
De-En IWSLT'14 dataset, 
removing the warmup stage
% stage raises the final 
increases the training loss 
% of a typical Transformer based neural machine translation model
from 3 to around 10, as shown in Figure~\ref{fig:eps2k}. 
Similar phenomena are observed in other scenarios like BERT (a bidirectional transformer language model) pre-training~\citep{devlin2018bert}.% \XD{Require a few words to explain BERT...}

% Since the theoretical underpinnings of the warmup are lacking,
Due to the lack of the theoretical underpinnings,
there is neither guarantee that warmup would 
% work for various machine learning settings 
bring consistent improvements for various machine learning settings 
nor guidance on how we should conduct warmup. 
Thus, researchers typically use different settings in different applications and have to take a trial-and-error approach,
% which can be tedious and time-consuming for many large-scale machine learning tasks, such as language modeling and image classification. 
which can be tedious and time-consuming. 

% In this paper, we conduct an in-depth analysis of the convergence issue, 
In this paper, we conduct both empirical and theoretical analysis of the convergence issue to identify its origin.
% In this paper, we conduct both empirical and theoretical analysis to identify the origin of the convergence issue.
%Here, we analyze this phenomenon, study why we need warmup, and explore the optimization issue it tries to address. 
%Specifically, 
We show that its root cause is: the adaptive learning rate has undesirably large variance in the early stage of model training, due to the limited amount of training samples being used. 
Thus, to reduce such variance, it is better to use smaller learning rates in the first few epochs of training, which justifies the warmup heuristic.
%We show that due to the limited sample size, the adaptive learning rate has undesirably large variance in the early stage of training, which obstructs the optimization process.

Inspired by our analysis results, we propose a new variant of Adam, called Rectified Adam (RAdam), which explicitly rectifies the variance of the adaptive learning rate based on derivations.
We conduct extensive experiments on language modeling, image classification, and neural machine translation.
RAdam brings consistent improvement over the vanilla Adam, which verifies the variance issue generally exists on various tasks across different network architectures.

In summary, our main contributions are two-fold:
% \begin{itemize}[leftmargin=*,topsep=0pt]
\begin{itemize}[leftmargin=*]
    \item
    \vspace{-0.1in}
    We identify the variance issue of the adaptive learning rate and present a theoretical justification for the warmup heuristic. 
    We show that the convergence issue is due to the undesirably large variance of the adaptive learning rate in the early stage of model training.
    \item 
    We propose a new variant of Adam (\ie, RAdam), which not only explicitly rectifies the variance and is theoretically sound, but also compares favorably with the heuristic warmup.
\end{itemize}

\section{Preliminaries and Motivations}
% \vspace{-0.1in}

\noindent
\textbf{Generic adaptive methods.} Algorithm~\ref{algo:adaptive} is a generic framework (all operations are element-wise).
It describes various popular stochastic gradient descent algorithms~\citep{reddi2019convergence}.
% like Adam~\citep{kingma2014adam} and, Adagrad~\citep{duchi2011adaptive} and RMSprop~\citep{tieleman2012lecture}.
Specifically, different optimization algorithms can be specified by different choices of $\phi(.)$ and $\psi(.)$, where
$\phi(.)$ specifies how the momentum at time step $t$ is calculated, and $\psi(.)$ how the adaptive learning rate at $t$ is calculated.
For example, 
% in the vanilla SGD, these two functions are set to:
% \begin{align*}
% \vspace{-0.1in}
% \phi(g_1, \cdots, g_t) = g_t \quad\mbox{and}\quad \psi(g_1, \cdots, g_t) = 1.
% \vspace{-0.2in}
% \end{align*}
% And 
in the Adam algorithm, we have:
\begin{align}
% \vspace{-0.2in}
\phi(g_1, \cdots, g_t) = \frac{(1 - \beta_1)\sum_{i = 1}^t \beta_1^{t-i}g_i}{1 - \beta_1^t} 
\quad\mbox{and}\quad 
\psi(g_1, \cdots, g_t) = \sqrt{\frac{1-\beta_2^t}{(1 - \beta_2)\sum_{i = 1}^t \beta_2^{t-i}g_i^2}}.
\label{eqn:adam}
%\vspace{0.1in}
\end{align}
For numerical stability, the function $\psi(.)$ in Equation~\ref{eqn:adam} is usually calculated as $\hat{\psi}(g_1, \cdots, g_t) = \frac{\sqrt{1-\beta_2^t}}{\epsilon + \sqrt{(1 - \beta_2)\sum_{i = 1}^t \beta_2^{t-i}g_i^2}}$, 
where $\epsilon$ is a relatively small / negligible value (\eg, $1\times 10^{-8}$). 
% \vspace{-0.1in}
% \setlength{\textfloatsep}{-1pt}

\begin{algorithm}[!ht]
\DontPrintSemicolon
\KwIn{$\{\alpha_t\}_{t = 1}^T$: step size, $\{\phi_t, \psi_t\}_{t = 1}^T$: function to calculate momentum and adaptive rate, \newline
$\theta_0$: initial parameter, $f(\theta)$: stochastic objective function.}
\KwOut{$\theta_T$: resulting parameters}
\While{$t = 1$ to $T$}{
    $g_t \gets \nabla_{\theta} f_t(\theta_{t - 1})$ (Calculate gradients w.r.t. stochastic objective at timestep t)\;
    $m_t \gets \phi_t (g_1, \cdots, g_t)$ (Calculate momentum)\;
    $l_t \gets \psi_t (g_1, \cdots, g_t)$ (Calculate adaptive learning rate)\;
    $\theta_t \gets \theta_{t-1} - \alpha_t m_t l_t$ (Update parameters)\;
}
\Return{$\theta_T$}
\caption{Generic adaptive optimization method setup. All operations are element-wise. }
\label{algo:adaptive}
\end{algorithm}

% \begin{figure}[!htb]
\begin{figure}[t]
\centering
\vspace{-0.3in}
    \includegraphics[width=\textwidth]{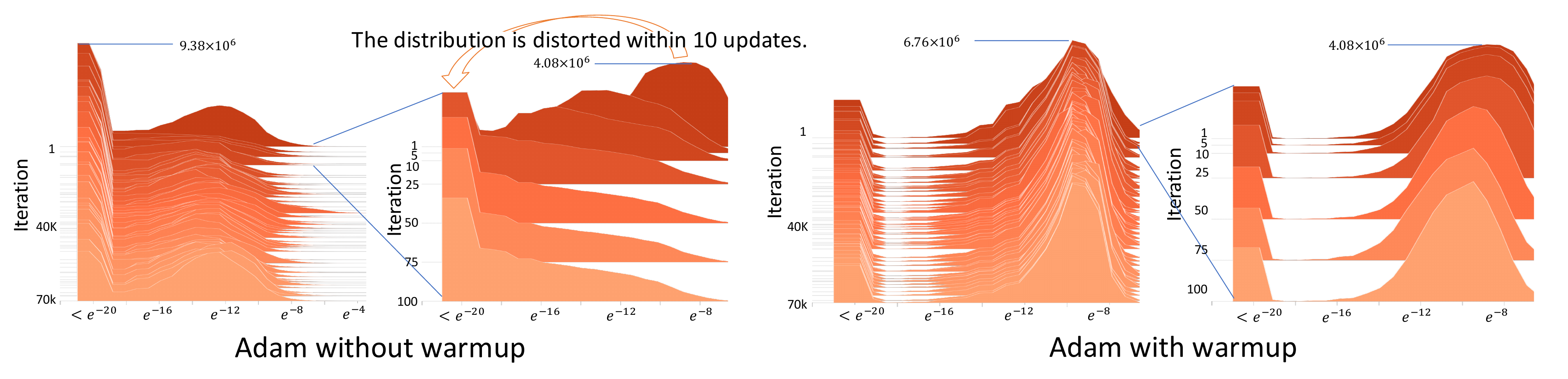}
\vspace{-0.6cm}
\caption{The absolute gradient histogram of the Transformers on the De-En IWSLT' 14 dataset during the training (stacked along the y-axis). X-axis is absolute value in the log scale and the height is the frequency. Without warmup, the gradient distribution is distorted in the first 10 steps. }
\vspace{-0.3cm}
\label{fig:histogram_2}
\end{figure}

\noindent
\textbf{Learning rate warmup.} Instead of setting the learning rate $\alpha_t$ as a constant or in a decreasing order, 
% a learning rate warmup strategy sets $\alpha_t$ as some small values in the first few steps.
a learning rate warmup strategy sets $\alpha_t$ as smaller values in the first few steps, thus not satisfying $\forall t\, \alpha_{t+1} \leq \alpha_{t}$.
% and decrease $\alpha_t$ after $T_w$ steps. 
For example, linear warmup sets $\alpha_t = t \,\alpha_0$ when $t < T_w$.
%In many machine learning applications, e.g., neural machine translation, it has been empirically observed that these adaptive stochastic optimization algorithms require a learning rate warmup stage, (\eg, a widely used warmup strategy is to set the learning rate as $\alpha_t = t \,\alpha_0$ in the first $T_w$ steps, \ie, $t < T_w$); otherwise, the loss is trapped in bad/suspicious local optima. S
Warmup has been demonstrated to be beneficial in many deep learning applications.
For example, in the NMT experiments in Figure~\ref{fig:eps2k}, the training loss convergences around 10 when warmup is not applied (Adam-vanilla), and it surprisingly decreases to below 3 after applying warmup (Adam-warmup).

To further analyze this phenomenon, we visualize the histogram of the absolute value of gradients on a log scale in Figure~\ref{fig:histogram_2}.
We observe that, without applying warmup, the gradient distribution is distorted to have a mass center in relatively small values within 10 updates. Such gradient distortion means that the vanilla Adam is trapped in bad/suspicious local optima after the first few updates. 
Warmup essentially reduces the impact of these problematic updates to avoid the convergence problem.
%Intuitively, since the widely used adaptive learning rates are designed as the inverse scale of the gradient (\eg, $\psi(.)$ in Equation~\ref{eqn:adam}), small gradients would have larger adaptive learning rates, and random fluctuations of these gradients would be amplified and hinder the optimizer from advancing model training. 
% Therefore, the vanilla Adam will be trapped in this stage after the first few updates
%Therefore, such gradient distortion can trap the vanilla Adam after the first few updates, and it is necessary to fix those problematic updates to avoid the convergence problems.
In the following sections, we focus our analysis on learning rate warmup for the Adam algorithm, while it can be applied to other algorithms that use similar adaptive learning rate ($\psi(.)$) designs, \eg, RMSprop~\citep{tieleman2012lecture} and Nadam~\citep{dozat2016incorporating}.

\vspace{-0.1cm}
\section{Variance of the Adaptive Learning Rate}
\label{sec:ana}
\vspace{-0.1cm}

In this section, we first introduce empirical evidence, then analyze the variance of the adaptive learning rate to support our hypothesis 
-- \emph{Due to the lack of samples in the early stage, the adaptive learning rate has an undesirably large variance, which leads to suspicious/bad local optima}.

% To begin with, we first analyze a special case.

To convey our intuition, we begin with a special case. 
When $t = 1$, we have $\psi(g_1) = \sqrt{1/g_1^2}$. 
We view $\{g_1, \cdots, g_t\}$ as i.i.d. Gaussian random variables following $\cN(0, \sigma^2)$\footnote{The mean zero normal assumption is valid at the beginning of the training, since weights are sampled from normal distributions with mean zero \citep{balduzzi2017shattered}, further analysis is conducted in Section~\ref{subsec:sma-ema}.}.
Therefore, $1/g_1^2$ is subject to the scaled inverse chi-squared distribution, $\sics(1, 1/\sigma^2)$, and $\Var[\sqrt{1/g_1^2}]$ is divergent.
% Noted 
% $\Var[\sqrt{1/g_1^2}] \propto \int_0^{\infty} x^{-1} e^{-x} dx$
% and it is divergent. 
It means that the adaptive ratio can be undesirably large in the first stage of learning.
Meanwhile, setting a small learning rate at the early stage can reduce the variance ($\Var[\alpha x] = \alpha^2\Var[x]$), thus alleviating this problem.
Therefore, we suggest it is the unbounded variance of the adaptive learning rate in the early stage
that causes the problematic updates. 

\begin{figure}[t]
\centering
    \includegraphics[width=0.8\textwidth]{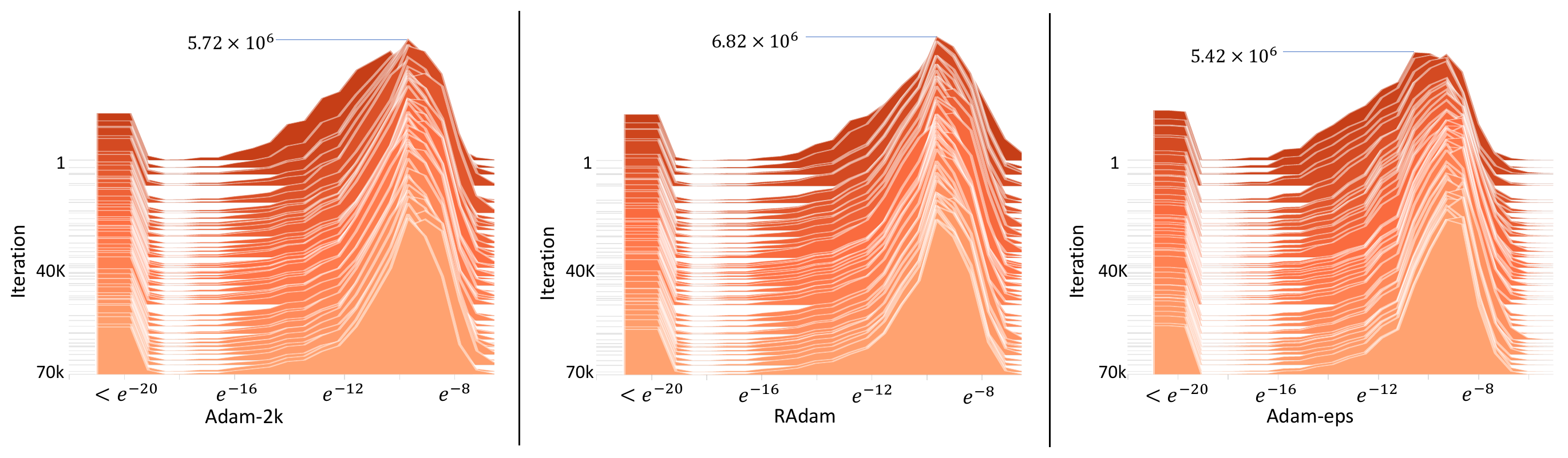}
\vspace{-0.3cm}
\caption{The histogram of the absolute value of gradients (on a log scale) during the training of Transformers on the De-En IWSLT' 14 dataset. using Adam-2k, RAdam and Adam-eps. }
\vspace{-0.4cm}
\label{fig:histogram_3}
\end{figure}

\vspace{-0.1cm}
\subsection{Warmup as Variance Reduction}
\vspace{-0.1cm}
In this section, we design a set of controlled experiments to verify our hypothesis. 
Particularly, we design two variants of Adam that reducing the variance of the adaptive learning rate: \textit{Adam-2k} and \textit{Adam-eps}. We compare them to vanilla Adam with and without warmup on the IWSLT'14 German to English translation dataset~\citep{cettolo2014report}. 

In order to reduce the variance of the adaptive learning rate ($\psi(.)$), Adam-2k only updates $\psi(.)$ in the first two thousand iterations, while the momentum ($\phi(.)$) and parameters ($\theta$) are fixed\footnote{Different from \cite{gotmare2018a}, all parameters and first moments are frozen in the first 2000 iterations.}; other than this, it follows the original Adam algorithm.
To make comparison with other methods, its iterations are indexed from -1999 instead of 1.
In Figure~\ref{fig:eps2k}, we observe that, after getting these additional two thousand samples for estimating the adaptive learning rate, Adam-2k avoids the convergence problem of the vanilla-Adam. 
Also, comparing Figure~\ref{fig:histogram_2} and Figure~\ref{fig:histogram_3}, getting large enough samples prevents the gradient distribution from being distorted. 
These observations verify our hypothesis that the lack of sufficient data samples in the early stage is the root cause of the convergence issue. 

Another straightforward way to reduce the variance is to increase the value of $\epsilon$ in $\hat{\psi}(g_1, \cdots, g_t) = \frac{\sqrt{1-\beta_2^t}}{\epsilon + \sqrt{(1 - \beta_2)\sum_{i = 1}^t \beta_2^{t-i}g_i^2}}$.
Actually, if we assume $\hat{\psi}(.)$ is subject to the uniform distribution, its variance equals to $\frac{1}{12 \epsilon^2}$.
Therefore, we design Adam-eps, which uses a non-negligibly large $\epsilon = 10^{-4}$, while $\epsilon = 10^{-8}$ for vanilla Adam.
Its performance is summarized in Figure~\ref{fig:eps2k}. 
We observe that it does not suffer from the serious convergence problem of vanilla-Adam. This further demonstrates that the convergence problem can be alleviated by reducing the variance of the adaptive learning rate, and also explains why tuning $\epsilon$ is important in practice \citep{liu2019roberta}. 
Besides, similar to Adam-2k, it prevents the gradient distribution from being distorted (as shown in Figure~\ref{fig:histogram_3}). 
% However, the experimental results in Figure~\ref{fig:eps2k} show that it produces a much worse performance comparing to Adam-2k and Adam-warmup. 
However, as in Figure~\ref{fig:eps2k}, it produces a much worse performance comparing to Adam-2k and Adam-warmup. 
We conjecture that this is because large $\epsilon$ induces a large bias into the adaptive learning rate and slows down the optimization process. 
Thus, we need a more principled and rigorous way to control the variance of the adaptive learning rate. 
In the next subsection, we will present a theoretical analysis of the variance of the adaptive learning rate.

\vspace{-0.1cm}
\subsection{Analysis of Adaptive Learning Rate Variance}
\vspace{-0.1cm}
As mentioned before, Adam uses the exponential moving average to calculate the adaptive learning rate. 
For gradients $\{g_1, \cdots, g_t\}$, their exponential moving average has a larger variance than their simple average. 
Also, in the early stage ($t$ is small), the difference of the exponential weights of $\{g_1, \cdots, g_t\}$ is relatively small (up to $1 - \beta_2^{t-1}$).
Therefore, for ease of analysis, we approximate the distribution of the exponential moving average as the distribution of the simple average~\citep{nau2014forecasting}, \ie, $p(\psi(.)) = p(\sqrt{\frac{1-\beta_2^t}{(1 - \beta_2)\sum_{i = 1}^t \beta_2^{t-i}g_i^2}}) \approx p(\sqrt{\frac{t}{\sum_{i=1}^t g_i^2}})$. Since $g_i \sim \gN(0, \sigma^2)$, we have $\frac{t}{\sum_{i=1}^t g_i^2} \sim \sics(t, \frac{1}{\sigma^2})$. 
Therefore, we assume $\frac{1-\beta_2^t}{(1 - \beta_2)\sum_{i = 1}^t \beta_2^{t-i}g_i^2}$ also subjects to a scaled inverse chi-square distribution with $\rho$ degrees of freedom (further analysis on this approximation is conducted in Section~\ref{subsec:sma-ema}). 
Based on this assumption, we can calculate $\Var[\psi^2(.)]$ and the PDF of $\psi^2(.)$.
Now, we proceed to the analysis of its square root variance, \ie, $\Var[\psi(.)]$, and show how the variance changes with $\rho$ (which corresponds to number of used training samples). 

\vspace{-0.1cm}
\begin{theorem}
If $\psi^2(.) \sim \sics(\rho, \frac{1}{\sigma^2})$, $\Var[\psi(.)]$ monotonically decreases as $\rho$ increases.
\label{theorem: variance_mono}
\end{theorem}
\vspace{-0.5cm}
\begin{proof}
% \vspace{-0.3in}
For $\forall\, \rho > 4$, we have:
\begin{align}
\Var[\psi(.)] =
% \Var[\sqrt{x}] = 
% \E[x] - \E[\sqrt{x}]^2 
\E[\psi^2(.)] - \E[\psi(.)]^2 
= \tau^2 (\frac{\rho}{\rho-2} - \frac{\rho \,2^{2\rho - 5}}{\pi}\gB(\frac{\rho-1}{2}, \frac{\rho-1}{2})^2),
% \; \forall\, \rho > 4
\label{eqn:analytic-sqrt-var}
% \vspace{-0.5in}
\end{align}
where $\gB(.)$ is the beta function. 
By analyzing the derivative of $\Var[\psi(.)]$, we know it monotonically decreases as $\rho$ increases. 
The detailed derivation is elaborated in the Appendix~\ref{app:proof_mono}. 
\end{proof}
\vspace{-0.3cm}

Theorem~\ref{theorem: variance_mono} gives a qualitative analysis of the variance of the adaptive learning rate.
It shows that, due to the lack of used training samples in the early stage, $\Var[\psi(.)]$ is larger than the late stage (Figure~\ref{fig:var_approx}).
To rigorously constraint the variance, we perform a quantified analysis on $\Var[\psi(.)]$ by estimating the degree of freedoms $\rho$. 

\vspace{-0.1cm}
\section{Rectified Adaptive Learning Rate}
\vspace{-0.1cm}
\label{sec:fix}

In the previous section, Equation~\ref{eqn:analytic-sqrt-var} gives the analytic form of $\Var[\psi(.)]$, where $\rho$ is the degree of freedoms. 
Here, we first give an estimation of $\rho$ based on $t$ to conduct a quantified analysis for $\Var[\psi(g_1, \cdots, g_t)]$, then we describe the design of the learning rate rectification, and  compare it to the heuristic warmup strategies.  

% \subsection{Estimation of $\rho$ and Variance}
\vspace{-0.1cm}
\subsection{Estimation of $\rho$}
\vspace{-0.1cm}

The exponential moving average (EMA) can be interpreted as an approximation to the simple moving average (SMA) in real application~\citep{nau2014forecasting}, \ie, 
\begin{align}
\vspace{-0.2in}
p\left(\frac{(1 - \beta_2)\sum_{i = 1}^t \beta_2^{t-i}g_i^2}{1-\beta_2^t} \right) \approx p\left(\frac{\sum_{i = 1}^{f(t, \beta_2)} g_{t+1-i}^2}{f(t, \beta_2)}\right).
\label{eqn:sma_ema}
\vspace{-0.2in}
\end{align}
where $f(t, \beta_2)$ is the length of the SMA which allows the SMA to have the same ``center of mass'' with the EMA.
In other words, $f(t, \beta_2)$ satisfies:
\begin{align}
\vspace{-0.7in}
% {(1 - \beta_2)}/{(1-\beta_2^t)} \sum_{i = 1}^t i \cdot \beta_2^{t-i} = 1 /{f(t, \beta_2)}\cdot\sum_{i = 1}^{f(t, \beta_2)} (t+1-i).
\frac{(1 - \beta_2)\sum_{i = 1}^t \beta_2^{t-i} \cdot i }{1-\beta_2^t} = \frac{\sum_{i = 1}^{f(t, \beta_2)} (t+1-i)}{f(t, \beta_2)}.
\label{eqn:sma_eql_ema}
\vspace{-0.7in}
\end{align}
% By solving the above equation, we have: 
By solving Equation~\ref{eqn:sma_eql_ema}, we have:
$f(t, \beta_2) = \frac{2}{1 - \beta_2} - 1 - \frac{2 t \beta_2^t}{1 - \beta_2^t}$.
In the previous section, we assume: $\frac{1-\beta_2^t}{(1 - \beta_2)\sum_{i = 1}^t \beta_2^{t-i}g_i^2} \sim \sics(\rho, \frac{1}{\sigma^2})$. 
Here, since $g_i \sim \gN(0, \sigma^2)$, we have $\frac{\sum_{i = 1}^{f(t, \beta_2)} g_{t+1-i}^2}{f(t, \beta_2)} \sim  \sics(f(t, \beta_2), \frac{1}{\sigma^2})$. 
Thus, Equation~\ref{eqn:sma_ema} views $\sics(f(t, \beta_2), \frac{1}{\sigma^2})$ as an approximation to $\sics(\rho, \frac{1}{\sigma^2})$.
Therefore, we treat $f(t, \beta_2)$ as an estimation of $\rho$. 
For ease of notation, we mark $f(t, \beta_2)$ as $\rho_t$. Also, we refer $\frac{2}{1 - \beta_2} - 1$ as $\rho_\infty$ (maximum length of the approximated SMA), due to the inequality $f(t, \beta_2) \leq \lim_{t \to \infty} f(t, \beta_2) = \frac{2}{1 - \beta_2} - 1$.

\vspace{-0.2cm}
\begin{algorithm}[t]
\DontPrintSemicolon
\KwIn{$\{\alpha_t\}_{t = 1}^T$: step size, 
$\{\beta_1, \beta_2\}$: decay rate to calculate moving average and moving 2nd moment,
$\theta_0$: initial parameter, $f_t(\theta)$: stochastic objective function.}
\KwOut{$\theta_t$: resulting parameters}
$m_0, v_0 \gets 0, 0$ (Initialize moving 1st and 2nd moment)\;
$\rho_\infty \gets 2/(1 - \beta_2) - 1$ (Compute the maximum length of the approximated SMA)\;
\While{$t = \{1, \cdots, T\}$}{
    $g_t \gets \nabla_{\theta} f_t(\theta_{t - 1})$ (Calculate gradients w.r.t. stochastic objective at timestep t)\;
    $v_t \gets \beta_2 v_{t - 1} + (1 - \beta_2) g_t^2$ (Update exponential moving 2nd moment)\;
    $m_t \gets \beta_1 m_{t - 1} + (1 - \beta_1) g_t$ (Update exponential moving 1st moment)\;
    $\widehat{m_t} \gets m_t / (1 - \beta_1^t)$ (Compute bias-corrected moving average)\;
    $\rho_t \gets \rho_\infty - 2t\beta_2^t/(1 - \beta_2^t)$(Compute the length of the approximated SMA)\;
    \uIf{the variance is tractable, \ie, $\rho_t > 4$}{
        $l_t \gets \sqrt{(1 - \beta_2^{t}) / v_t}$ (Compute adaptive learning rate)\;
        $r_t \gets \sqrt{\frac{(\rho_t - 4)(\rho_t - 2)\rho_\infty}{(\rho_\infty - 4)(\rho_\infty - 2)\rho_t}}$ (Compute the variance rectification term)\;
        $\theta_t \gets \theta_{t-1} - \alpha_t r_t \widehat{m_t} l_t$ (Update parameters with adaptive momentum)\;
    }
    \Else{
        $\theta_t \gets \theta_{t-1} - \alpha_t \widehat{m_t}$ (Update parameters with un-adapted momentum)\;
    }
}
\Return{$\theta_T$}
\caption{Rectified Adam. All operations are element-wise. }
% \vspace{-0.3cm}
\label{algo:cadam}
\end{algorithm}

\subsection{Variance Estimation and Rectification}
\vspace{-0.1cm}
Based on previous estimations, we have $\Var[\psi(.)] = \tau^2 (\frac{\rho_t}{\rho_t-2} - \frac{\rho_t \,2^{2\rho_t - 5}}{\pi}\gB(\frac{\rho_t-1}{2}, \frac{\rho_t-1}{2})^2)$.
The value of this function in the early stage is significantly larger than the late stage (as analyzed later, it decays roughly at the speed of $O(\frac{1}{\rho_t
})$). 
For example, 
% as in Figure~\ref{fig:var_approx}, 
the variance at $\rho_t = 5$ is over $100$ times larger than the variance at $\rho_t = 500$.
% \XD{Introducing Figure 7 requires explanation. How about just say we will analyze it in section 5.3.}
Additionally, based on Theorem~\ref{theorem: variance_mono}, we know $\min_{\rho_t} \Var[\psi(.)] = \Var[\psi(.)]|_{\rho_t = \rho_\infty}$ and mark this minimal value as $C_{\mbox{var}}$.
In order to ensure that the adaptive learning rate ($\psi(.)$) has consistent variance, we rectify the variance at the $t$-th timestamp as below,
% (with $\rho_t$), \ie,
\begin{align*}
% \vspace{-0.1in}
\Var[r_t \,\psi(g_1, \cdots, g_t)] = C_{\mbox{var}}
\quad
\mbox{where}
\quad
r_t = \sqrt{{C_{\mbox{var}}}/{\Var[\psi(g_1, \cdots, g_t)]}}.
% \label{eqn:analytic-rectification}
% \vspace{-0.1in}
\end{align*}
% \= \sqrt{\frac{(\rho_t - 4)(\rho_t - 2)\rho_\infty}{(\rho_\infty - 4)(\rho_\infty - 2)\rho_t}}.
Although we have the analytic form of $\Var[\psi(.)]$ (\ie, Equation~\ref{eqn:analytic-sqrt-var}), it is not numerically stable. 
Therefore, we use the first-order approximation to calculate the rectification term.
Specifically, by approximating $\sqrt{\psi^2(.)}$ to the first order~\citep{wolter2007taylor}, 
% Following~\citet{wolter2007taylor}, we have:
\begin{align*}
% \vspace{-0.2in}
\sqrt{\psi^2(.)} \approx \sqrt{\E[\psi^2(.)]} + \frac{1}{2\sqrt{\E[\psi^2(.)]}}(\psi^2(.) - \E[\psi^2(.)])
\quad\mbox{and}\quad
\Var[\psi(.)] \approx \frac{\Var[\psi^2(.)]}{4\E[\psi^2(.)]}.
% \label{eqn:sqrt-var}
% \vspace{-0.05in}
\end{align*}
Since $\psi^2(.) \sim \sics(\rho_t, \frac{1}{\sigma^2})$, we have:
% \vspace{-0.05in}
\begin{align}
% \vspace{-0.3in}
\Var[\psi(.)] \approx {\rho_t}/[{2(\rho_t - 2) (\rho_t - 4) \sigma^2}].
\label{eqn:variance_appro}
% \vspace{-0.1in}
\end{align}
% $
% \Var[\psi(.)] \approx {\rho_t}/[{2(\rho_t - 2) (\rho_t - 4) \sigma^2}].
% $
In Section~\ref{subsec:approx}, we conduct simulation experiments to examine Equation~\ref{eqn:variance_appro} and find that it is a reliable approximation. 
Based on Equation~\ref{eqn:variance_appro}, we know that $\Var[\sqrt{\psi(.)}]$ decreases approximately at the speed of $O(\frac{1}{\rho_t})$.
With this approximation, we can calculate the rectification term as:
% \vspace{-0.05in}
\begin{align*}
% \vspace{-0.3in}
r_t = \sqrt{\frac{(\rho_t - 4)(\rho_t - 2)\rho_\infty}{(\rho_\infty - 4)(\rho_\infty - 2)\rho_t}}.
% \vspace{-0.3in}
\end{align*}
% \vspace{-0.2in}
Applying our rectification term to Adam, we come up with a new variant of Adam, Rectified Adam (RAdam), as summarized in Algorithm~\ref{algo:cadam}. 
Specifically, when the length of the approximated SMA is less or equal than 4, the variance of the adaptive learning rate is intractable and the adaptive learning rate is inactivated. 
Otherwise, we calculate the variance rectification term and update parameters with the adaptive learning rate.
It is worth mentioning that, if $\beta_2 \leq 0.6$, we have $\rho_\infty \leq 4$ and RAdam is degenerated to SGD with momentum. 

\begin{figure}[t]
\begin{minipage}{0.69\textwidth}
% \vspace{-0.5cm}
\includegraphics[width=\textwidth]{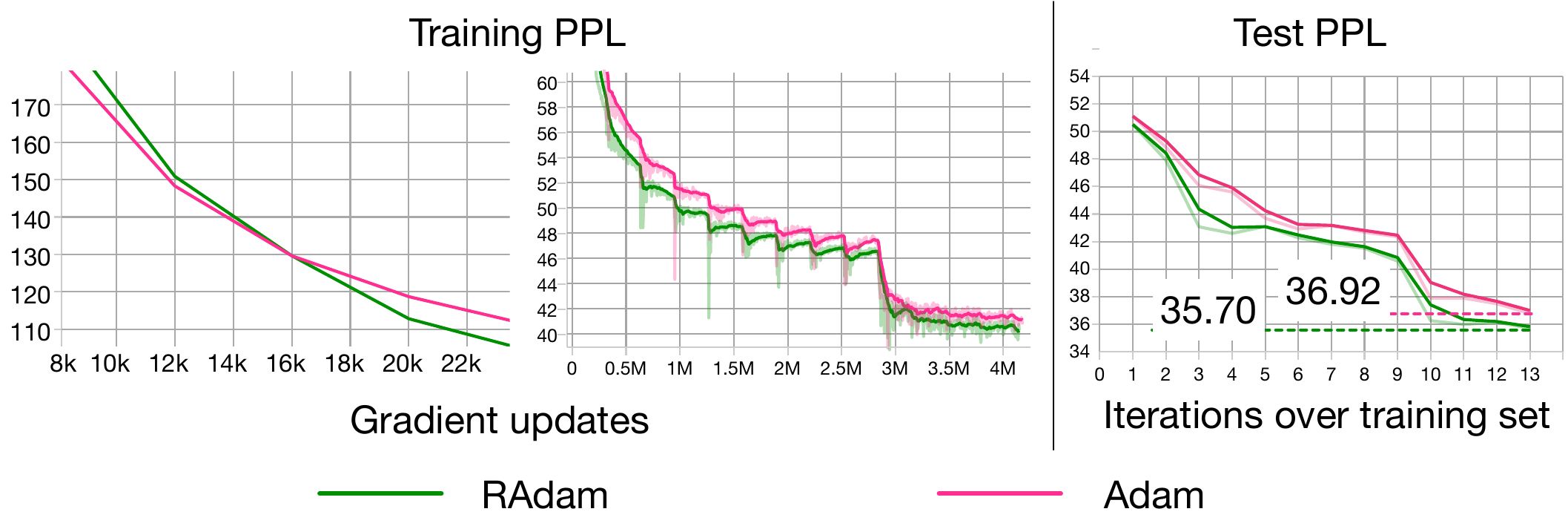}
\vspace{-0.7cm}
 \captionof{figure}{Language modeling (LSTMs) on the One Billion Word.}
    \label{fig:1bw}
% \vspace{-0.4cm}
\end{minipage}
\begin{minipage}{0.3\textwidth}
    \centering
    % \begin{tabular}{c|c|c}
    %      Method &  Cifar10 & ImageNet\\
    %      \hline
    %      SGD   & 91.51 & 69.86 \\
    %      Adam   & 90.54 & 66.54 \\
    %      RAdam   & 91.38 & 67.62 \\
    % \end{tabular}
    \captionof{table}{Image Classification}
    \vspace{-0.2cm}
    \begin{tabular}{c|c|c}
        %  Method &  Cifar10 & ImageNet\\
        %  \hline
        & Method & Acc.\\
        \hline
        & & \\[-7pt]
      \parbox[t]{2mm}{\multirow{3}{*}{\rotatebox[origin=c]{90}{\small \;CIFAR10}}}
      & SGD & 91.51 \\[2pt]
        & Adam   & 90.54 \\[2pt]
        & RAdam   & 91.38 \\[2pt]
        \hline
        & & \\[-7pt]
        \parbox[t]{2mm}{\multirow{3}{*}{\rotatebox[origin=c]{90}{\small \;ImageNet}}}
        & SGD & 69.86 \\[2pt]
        & Adam & 66.54 \\[2pt]
        & RAdam & 67.62 
    \end{tabular}
% \end{table}
\end{minipage}
\end{figure}
\vspace{-0.4cm}
\begin{figure}[t]
\centering
\vspace{-0.2cm}
\includegraphics[width=\textwidth]{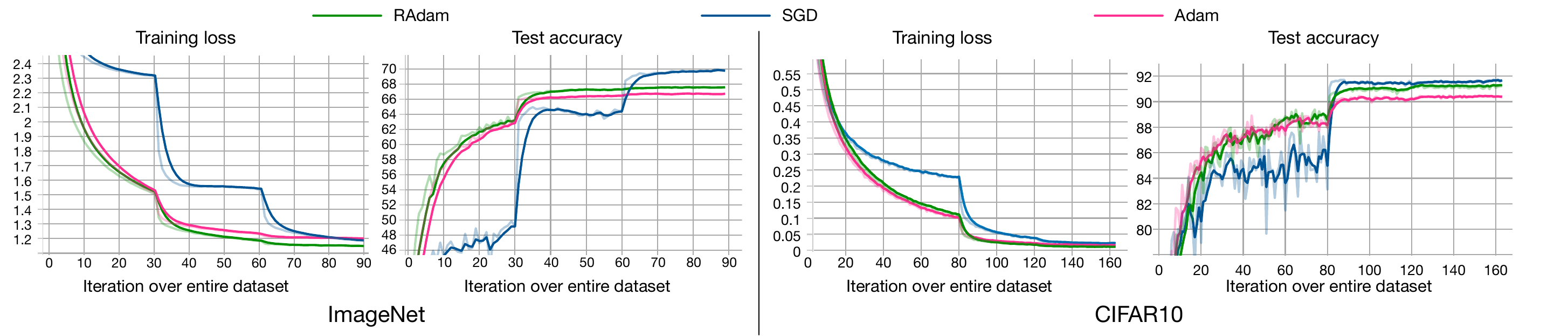}
\vspace{-0.6cm}
\caption{Training of ResNet-18 on the ImageNet and ResNet-20 on the CIFAR10 dataset.}
\vspace{-0.2cm}
    \label{fig:cifa10}
\end{figure}

\subsection{In Comparison with Warmup and Other Stabilization Techniques}
Different from the analysis in this paper,  warmup is originally proposed to handle training with very large batches for SGD~\citep{goyal2017accurate,gotmare2018a,bernstein2018signsgd,Xiao2017DSCOVRRP}.
We notice that $r_t$ has a similar form to the heuristic linear warmup, which can be viewed as setting the rectification term as $\frac{min(t, T_w)}{T_w}$. 
It verifies our intuition that warmup works as a variance reduction technique. 
% Comparing these two strategies, 
RAdam deactivates the adaptive learning rate when its variance is divergent, thus avoiding undesired instability in the first few updates. 
Besides, our method does not require an additional hyperparameter (\ie, $T_w$) 
% to control the variance reduction 
and can automatically adapt to different moving average rules. 

% In this paper, we identify and fix an underlying issue of adaptive optimization methods instead of neural architectures.
% In this paper
Here, we identify and address an underlying issue of adaptive optimization methods independent of (neural) model architectures.
Thus, the proposed rectification term is orthogonal to other training stabilization techniques such as  gradient clipping~\citep{bengio2013advances}, smoothing the adaptive learning rate (\ie, increasing $\epsilon$, applying geometric mean filter~\citep{chen2018closing}, or adding range constraints~\citep{luo2019adaptive}), initialization~\citep{balduzzi2017shattered,zhang2019fixup} and normalization~\citep{ba2016layer,ioffe2015batch}. 
% Indeed, these techniques can be integrated with our proposed variance rectification. 
Indeed, these techniques can be combined with the proposed variance rectification method.

\vspace{-0.1cm}
\section{Experiments}
\vspace{-0.1cm}

%In this section, we conduct experiments to verify our intuition and examine RAdam, while detailed hyperparameter settings are elaborated in the Appendix.
We evaluate RAdam on several benchmarks:
% due to the space limitation.}: 
One Billion Word for language modeling; Cifar10 and ImageNet for image classification; IWSLT'14 De-En/EN-DE and WMT'16 EN-De for neural machine translation.
% ; Pong for Reinforcement Learning. 
Following~\cite{loshchilov2017fixing}, we decouple weight decays in the vanilla Adam, Adam with warmup and RAdam 
% employs the fixed weight decay 
in our experiments.  Details are in Appendix~\ref{app:implement}.
%It is worth noticing that, the vanilla Adam used in our experiments employs the fixed weight decay~\citep{loshchilov2017fixing}. 

\vspace{-0.1cm}
\subsection{Comparing to Vanilla Adam}
\vspace{-0.1cm}
As analyzed before, the adaptive learning rate has undesirably large variance in the early stage of training and leads to suspicious/bad local optima on NMT. 
% One question we are interested in answering is: whether such an issue widely exits in other similar tasks and applications.
One question we are interested in is: whether such an issue widely exits in other similar tasks and applications.
%Here, we first verify whether such an issue widely exists by comparing the vanilla Adam with RAdam. 
%Specifically, we conduct experiments on 
Thus, we conduct a set of experiments with two classical tasks of NLP and CV, \ie, language modeling and image classification.
% are used in this study.
RAdam not only results in consistent improvements over the vanilla Adam, but also demonstrates its robustness to the change of learning rates. 
It verifies that the variance issue exists in various machine learning applications, and has a big impact on the model behavior. 
% Detailed comparison and analysis are described as follows. 

% \begin{figure}[t]
% \centering
% \vspace{-0.2cm}
% \includegraphics[width=\textwidth]{fig/cifa_imagenet.pdf}
% \vspace{-0.6cm}
% \caption{Training of ResNet-18 on the ImageNet and ResNet-20 on the CIFAR10 dataset.}
% \vspace{-0.4cm}
%     \label{fig:cifa10}
% \end{figure}

% \begin{wraptable}{r}{0.3\textwidth}
%     \centering
% \vspace{-0.3cm}
%     \caption{Perplexity on Language Modeling}
% \vspace{-0.2cm}
%     \label{tab:lm_ppl}
%     \begin{tabular}{c|c}
%          Method &  One Billion Word\\
%          \hline
%          Adam   & 36.92 \\
%          RAdam   &  35.70 \\
%     \end{tabular}
% \vspace{-0.5cm}
% \end{wraptable}

\noindent
\textbf{Performance Comparison.}
The performances on language modeling (\ie, One Billion Word~
%\footnote{Rare words that occur less than 3 times are replaced with a special token, the resulting dictionary is shrank from 7.9M to 6.4M.}
\citep{Chelba2013OneBW}) and image classification (\ie, CIFAR10~\citep{krizhevsky2009learning} and ImageNet~\citep{deng2009imagenet}) are 
%summarized in Table~\ref{tab:lm_ppl} and Table~\ref{tab:image_cls}, and their learning curves are
presented in Figure~\ref{fig:1bw}, \ref{fig:cifa10}. %and Table~\ref{tab:image_cls}.%, respectively. 
The results show that RAdam outperforms Adam in all three datasets.
As shown in Figure~\ref{fig:1bw}, although the rectification term makes RAdam slower than the vanilla Adam in the first few epochs, it allows RAdam to converge faster after that. %have a faster speed after that. 
In other words, by reducing the variance of the adaptive learning rate in the early stage, it gets both faster convergence and better performance, which verifies the impact of the variance issue. 
%At the same time, RAdam demonstrates consistent improvements over Adam on image classification. 
We also observe that RAdam obtains consistent improvements over Adam on image classification.
It is worth noting that, on both ImageNet and CIFAR10, although RAdam fails to outperform SGD in terms of test accuracy, it results in a better training performance (\eg, the training accuracy of SGD, Adam, and RAdam on ImageNet are $69.57$, $69.12$ and $70.30$ respectively). 
%[Jianfeng: be careful to draw such a conclusion here. this might be a sign of overfitting, or worse generalization capability.]

% \begin{wraptable}{r}{0.35\textwidth}
%     \centering
% \vspace{-0.3cm}
%     \caption{Image Classification}
% \vspace{-0.3cm}
%     \label{tab:image_cls}
%     \begin{tabular}{c|c|c}
%         \multirow{2}{*}{Method}        & \multicolumn{2}{c}{Accuracy} \\
%         \cline{2-3}
%                 &  CIFAR10 & ImageNet\\
%          \hline
%          SGD   & 91.51 & 69.86 \\
%          Adam   & 90.54 & 66.54 \\
%          RAdam   & 91.38 & 67.62 
%     \end{tabular}
% % \vspace{-0.3cm}
% \end{wraptable}

% \begin{figure}[h]
% \centering
% \vspace{-0.2cm}
% \includegraphics[width=0.8\textwidth]{fig/one_billion.pdf}
% \caption{Training of LSTMs on the One Billion Word dataset.}
%     \label{fig:1bw}
% \end{figure}

\noindent
\textbf{Robustness to Learning Rate Change.}
Besides performance improvements, RAdam also improves the robustness of model training. 
%Specifically,
We use different initial learning rates, conduct experiments with ResNet-20 on the CIFAR10 datasets, and summarize their performance in Figure~\ref{fig:cifa10_lr}. 
For learning rates within a broad range (\ie, $\{0.1, 0.03, 0.01, 0.003\}$), RAdam achieves consistent model performances (their test accuracy curves highly overlap with each other), while Adam and SGD are shown to be more sensitive to the learning rate. 
The observation can be interpreted that by rectifying the variance of the adaptive learning rate, RAdam improves the robustness of model training and can adapt to different learning rates of a broader range. 

\begin{figure}[t]
\centering
\vspace{0.2cm}
    \includegraphics[width=0.94\textwidth]{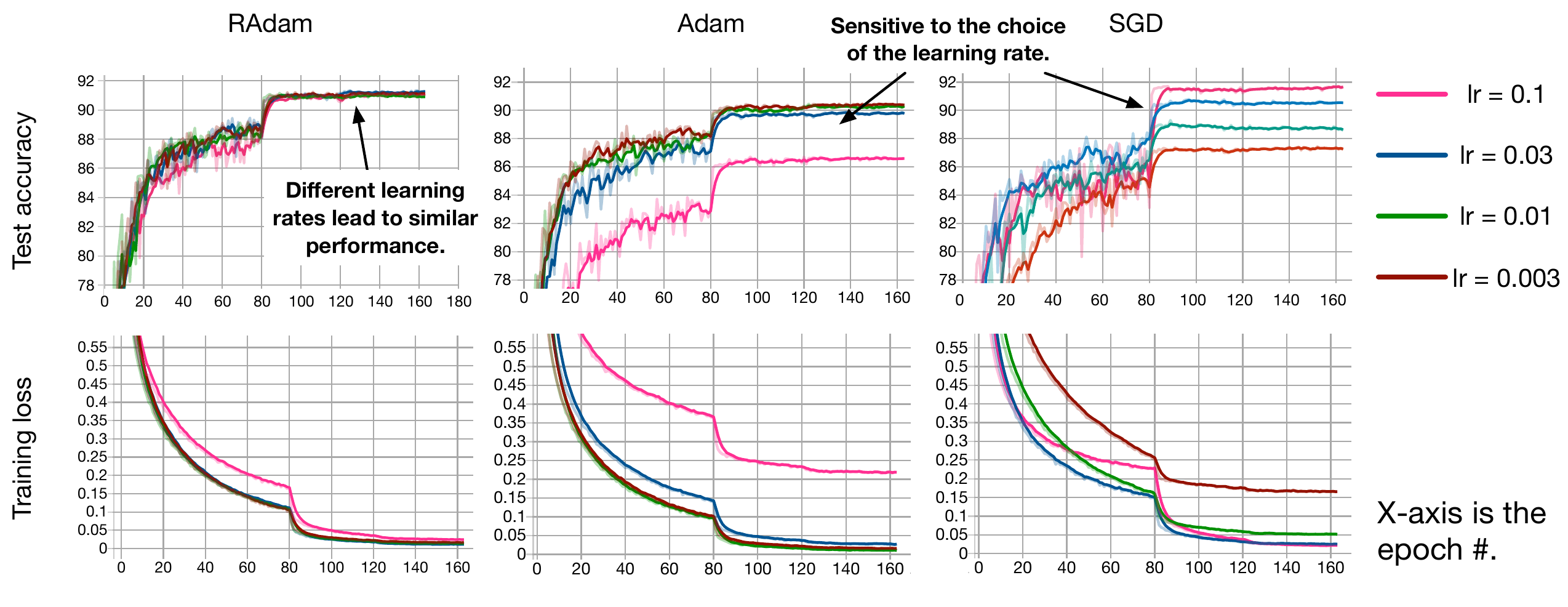}
    \vspace{-0.5cm}
\caption{Performance of RAdam, Adam and SGD with different learning rates on CIFAR10.}
% X-axis is the number of epochs. }
% The upper row is the test accuracy and the lower row is the training loss. }
% 
% \vspace{-0.05cm}
\label{fig:cifa10_lr}
\end{figure}

\begin{figure}[t]
\vspace{-0.2cm}
\centering
    \includegraphics[width=\textwidth]{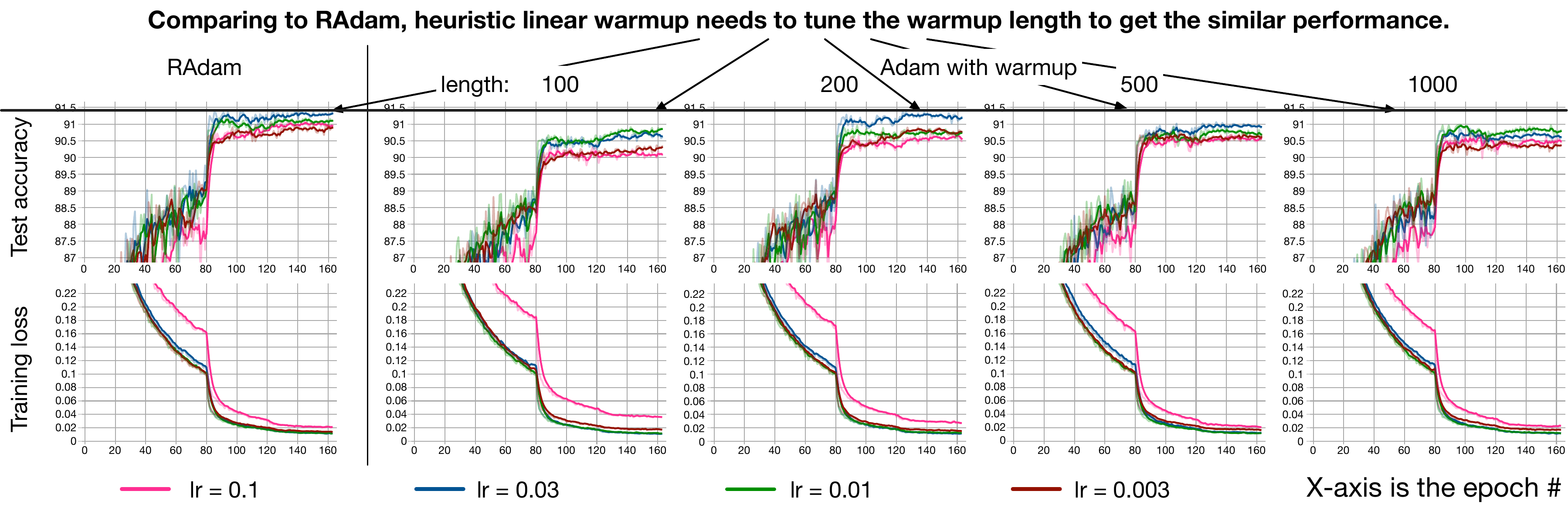}
\vspace{-0.7cm}
\caption{Performance of RAdam, Adam with warmup on CIFAR10 with different learning rates.}
% X-axis is the number of epochs. }
% The upper row is the test accuracy and the lower row is the training loss. }
\vspace{-0.5cm}
\label{fig:cifa10_wu}
\end{figure}

\vspace{-0.1cm}
\subsection{Comparing to Heuristic Warmup}
\vspace{-0.1cm}
To examine the effectiveness of RAdam, we first conduct comparisons on neural machine translation, on which the state-of-the-art employs Adam with the linear warmup.
Specifically, we conduct experiments on three datasets, i.e., IWSLT'14 De-En, IWSLT'14 En-De, and WMT'16 En-De. 
Due to the limited size of the IWSLT'14 dataset, we conduct experiments using 5 different random seeds and report their mean and standard derivation. 
As discussed before, the vanilla Adam algorithm leads to suspicious/bad local optima (i.e., converges to a training perplexity around 500), and needs a learning rate warmup stage to stabilize the training. 

We summarize the performance obtained with the heuristic warmup and our proposed rectification term in Table~\ref{tab:nmt} and visualize the training curve of IWSLT De-En in Figure~\ref{fig:eps2k}.
With a consistent adaptive learning rate variance, our proposed method achieves similar performance to that of previous state-of-the-art warmup heuristics. 
It verifies our intuition that the problematic updates of Adam are indeed caused by the undesirably large variance in the early stage. 

\begin{table}[t]
    \centering
\vspace{-0.2cm}
    \caption{BLEU score on Neural Machine Translation. 
    }
\vspace{-0.2cm}
    \label{tab:nmt}
    \begin{tabular}{c|c|c|c}
         Method & IWSLT'14 DE-EN & IWSLT'14 EN-DE & WMT'16 EN-DE \\
         \hline
         Adam with warmup  & $34.66 \pm 0.014$ & $28.56 \pm 0.067$ & $27.03$\\
         RAdam   & $34.76 \pm 0.003$ & $28.48 \pm 0.054$ & $27.27$\\
    \end{tabular}
\vspace{-0.6cm}
\end{table}

Moreover, we applied Adam with warmup on the CIFAR10 dataset. 
Its best accuracy on the test set is $91.29$, which is similar to RAdam ($91.38$). 
However, we found that RAdam requires less hyperparameter tuning. 
Specifically, we visualize their learning curves in Figure~\ref{fig:cifa10_wu}.
For some warmup steps, Adam with warmup is relatively more sensitive to the choice of the learning rate. 
RAdam, at the same time, is not only more robust, but also can automatically control the warmup behavior (\ie, without requiring the length of warmup). 
For example, when setting the learning rate as $0.1$, Adam with 100 steps of warmup fails to get satisfying performance and only results in an accuracy of $90.13$; RAdam successfully gets an accuracy of $91.06$, with the original setting of the moving average calculation (\ie, $\beta_1 = 0.9, \beta_2 = 0.999$). 
We conjecture the reason is due to the fact that RAdam, which is based on a rigorous variance analysis, explicitly avoids the extreme situation where the variance is divergent, and rectifies the variance to be consistent in other situations. 

% !TEX encoding = UTF-8
% !TEX Root = 0_main.tex

% \section{Simulated Verification}
% \vspace{-0.3cm}
\subsection{Simulated Verification}
% \vspace{-0.2cm}

% In this section, we examine the two approximations used in the previous section to analyze the variance.
% Specifically, as in Equation~\ref{eqn:variance_appro_appro}, we approximate $\Var[\sqrt{\frac{t}{\sum_{i=1}^t g_i^2}}]$ to the first order and assume $\psi^2(.) = \frac{1-\beta_2^t}{(1 - \beta_2)\sum_{i = 1}^t \beta_2^{t-i}g_i^2}$ subjects to the scaled inverse chi-square distribution. 
% In our analysis 
In Sections~\ref{sec:ana} and \ref{sec:fix}, we approximate $\Var[\sqrt{t/\sum_{i=1}^t g_i^2}]$ to the first order, and assume $\psi^2(.) = \frac{1-\beta_2^t}{(1 - \beta_2)\sum_{i = 1}^t \beta_2^{t-i}g_i^2}$ subjects to a scaled inverse chi-square distribution (this assumption covers the approximation from EMA to SMA).
Here, we examine these two approximations using simulations.

\noindent\textbf{First Order Approximation of $\Var[\sqrt{t/\sum_{i=1}^t g_i^2}]$.}
\label{subsec:approx}
% Here, we first derive the analytic form of $\Var[\sqrt{\frac{t}{\sum_{i=1}^t g_i^2}}]$ then compare its value to the first order approximation in Equation~\ref{eqn:variance_appro}.
% For ease of notation, we refer $\frac{t}{\sum_{i=1}^t g_i^2}$ as $x$, and mark $\frac{1}{\sigma^2}$ as $\tau^2$.
% Therefore, $x \sim \sics(t, \tau^2)$ and $p(x) = \frac{(\tau^2t/2)^{t/2}}{\Gamma(t/2)}\frac{\exp[\frac{-v\tau^2}{2x}]}{x^{1 + t/2}}$.
% We have: 
% \begin{equation}
%     \E[\sqrt{x}] = \int_{0}^{\infty} \sqrt{x} p(x) dx = \frac{\tau \sqrt{t} \,\Gamma(t/2 - 1)}{\sqrt{2} \,\Gamma(t/2)}.
%     \label{eqn:expect-sqrt-x}
% \end{equation}
% Based on Equation~\ref{eqn:inv-chi-sq} and \ref{eqn:expect-sqrt-x}, we have:
% \begin{equation}
% \Var[\sqrt{x}] = \E[x] - \E[\sqrt{x}]^2 = \tau^2 (\frac{t}{t-2} - \frac{t \,2^{2t - 5}}{\pi}\gB(\frac{t-1}{2}, \frac{t-1}{2})^2).
% \label{eqn:analytic-sqrt-var}
% \end{equation}
To compare Equations~\ref{eqn:variance_appro} and \ref{eqn:analytic-sqrt-var}, we assume $\tau = 1$ and plot their values and difference for $\nu = \{5, \cdots, 500\}$ in Figure~\ref{fig:var_approx}.
The curve of the analytic form and the first-order approximation highly overlap, and their difference is much smaller than their value.
This result verifies that our first-order approximation is very accurate.% of  $\Var[\sqrt{\frac{t}{\sum_{i=1}^t g_i^2}}]$.
% It is worth mentioning that, although we get the analytic form of $ \Var[\sqrt{\frac{t}{\sum_{i=1}^t g_i^2}}]$, it's preferable to use its first order approximation,
% % as Equation~\ref{eqn:variance_appro}, since the Equation~\ref{eqn:analytic-sqrt-var} is not numerical stable.
% which is more stable and efficient to calculate. 
% % (Mathematica\footnote{Version Number ?} fails to calculate the analytic value for $t = 1000$). 

\noindent\textbf{Scaled Inverse Chi-Square Distribution Assumption.}
\label{subsec:sma-ema}
In this paper, we assume $g_i$ accords to a Normal distribution with a zero mean. 
%Also, based on the similarity between the exponential moving average and simple moving average, 
We also assume $\psi^2(.)$ accords to the scaled inverse chi-square distribution to derive the variance of $\Var[\psi(.)]$,
based on the similarity between the exponential moving average and simple moving average.
Here, we empirically verify this assumption.
% via simulation. 

Specifically, since $g_i$ in the optimization problem may not be zero-mean, we assume its expectation is $\mu$ and sample $g_i$ from $\cN(\mu, 1)$.
Then, based on these samples, we calculate the variance of the original adaptive learning rate and the proposed rectified adaptive learning rate, \ie,  $\Var[\frac{1}{\widehat{v_t}}]$ and $\Var[\frac{r_t}{\widehat{v_t}}]$ respectively.
We set $\beta_2$ to $0.999$, the number of sampled trajectories to $5000$, the number of iterations to $6000$, and summarize the simulation results in Figure~\ref{fig:vt_var_simulation}. 
Across all six settings with different $\mu$, the adaptive learning rate has a larger variance in the first stage and the rectified adaptive learning rate has relative consistent variance. 
This verifies the reliability of our assumption. 
\section{Conclusion}
% \vspace{-0.4cm}
% In this paper, we analysis the underlying principles of warmup on the adaptive learning rate.
In this paper, we explore the underlying principle of the effectiveness of the warmup heuristic used for adaptive optimization algorithms. 
Specifically, we identify that, due to the limited amount of samples in the early stage of model training, the adaptive learning rate has an undesirably large variance and can cause the model to converge to suspicious/bad local optima. 
We provide both empirical and theoretical evidence to support our hypothesis, and further propose a new variant of Adam, whose adaptive learning rate is rectified so as to have a consistent variance. 
Empirical results demonstrate the effectiveness of our proposed method. 
% In future work, we plan to apply the proposed method to other applications such as Named Entity Recognition~\citep{lin2019reliability}.
% Another interesting direction to pursue is to adapt the choice of $\beta$ based on the variance estimation of different parameters, \ie, use a larger $\beta$ for parameters with a larger variance. 
In future work, we plan to replace the rectification strategy by sharing the second moment estimation across similar parameters. 

\vspace{-0.2cm}
\section*{Acknowledge}
\vspace{-0.3cm}
We thank Zeyuan Allen-Zhu for valuable discussions and comments, Microsoft Research Technology Engineering team for setting up GPU machines.
Research was sponsored in part by DARPA No. W911NF-17-C-0099 and FA8750-19-2-1004, National Science Foundation IIS 16-18481, IIS 17-04532, and IIS-17-41317, and DTRA HDTRA11810026. 
% Any opinions, findings, and conclusions or recommendations expressed in this document are those of the authors and should not be interpreted as the views of any U.S. Government. 
% The U.S. Government is authorized to reproduce and distribute reprints for Government purposes notwithstanding any copyright notation hereon.

\begin{figure}[t]
\begin{minipage}{0.38\textwidth}
\begin{figure}[H]
\centering
    \includegraphics[width=\textwidth]{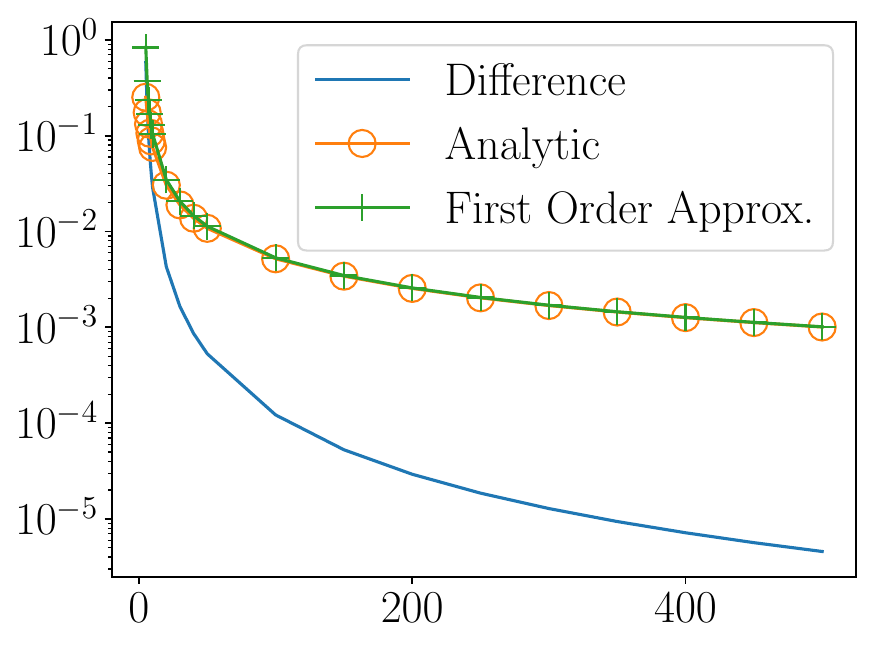}
\vspace{-0.6cm}
\caption{The value of Equation~\ref{eqn:analytic-sqrt-var}, Equation~\ref{eqn:variance_appro} and their difference (absolute difference). The x-axis is $\rho$ and the y-axis is the variance (log scale).}
    \label{fig:var_approx}
\end{figure}
\end{minipage}
\;
\begin{minipage}{0.61\textwidth}
\begin{figure}[H]
\centering
\begin{tabular}[H]{ ccc } 
 \includegraphics[width=0.3\textwidth]{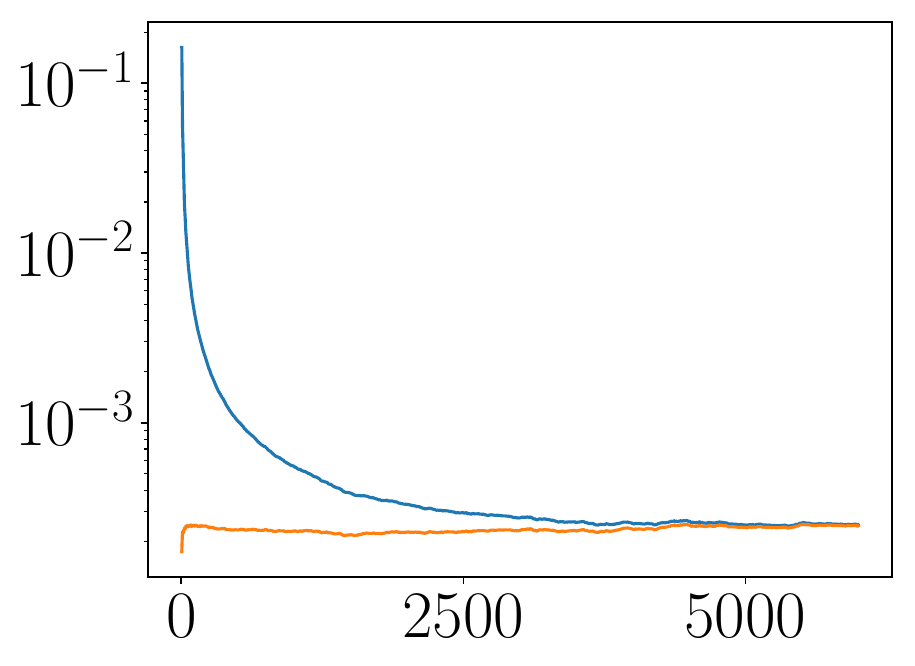} & 
 \includegraphics[width=0.3\textwidth]{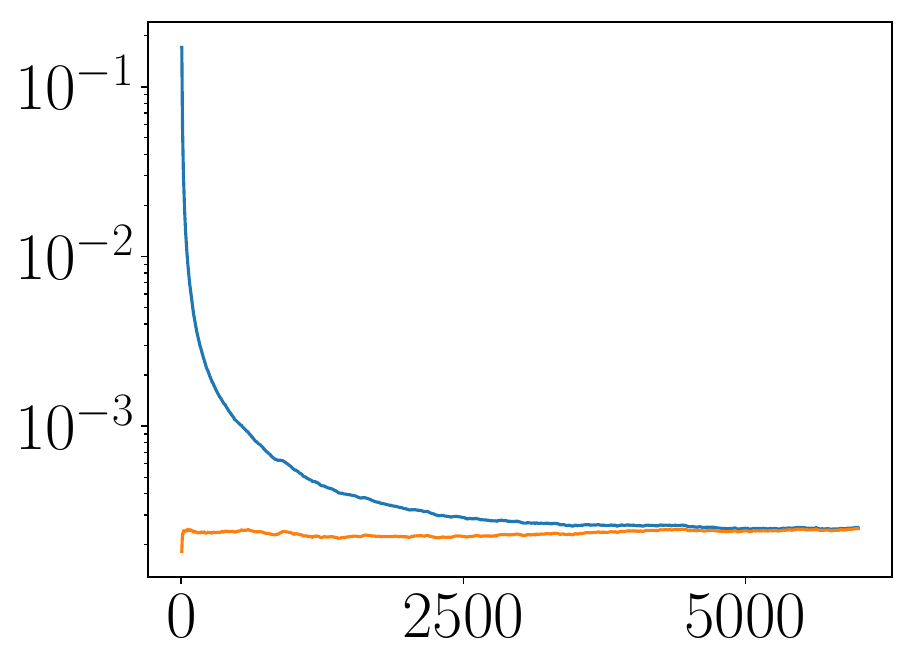} & 
 \includegraphics[width=0.3\textwidth]{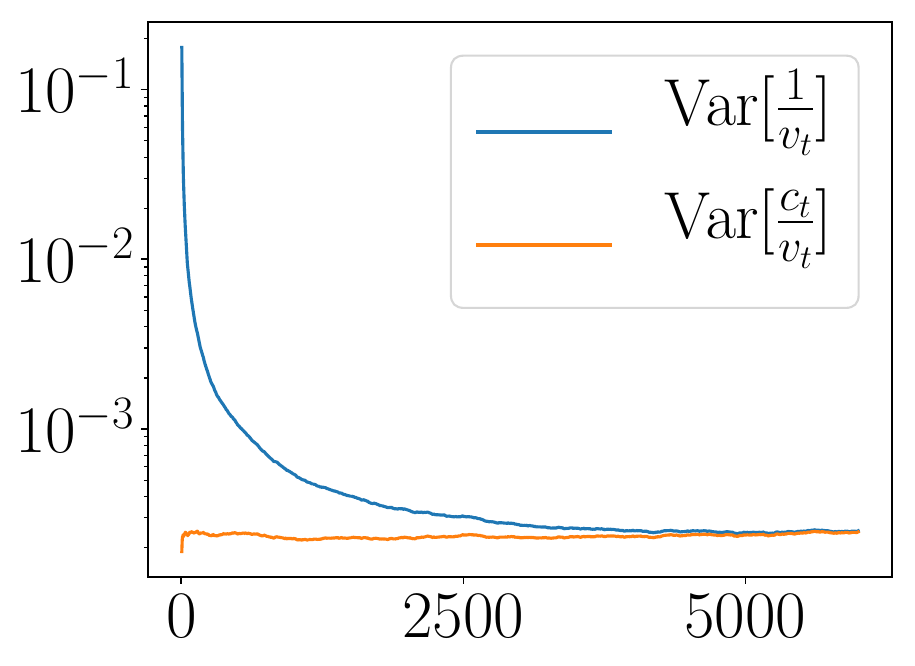}  \\
 $\mu=0$ & $\mu=0.001$  & $\mu=0.01$ \\ 
 \includegraphics[width=0.3\textwidth]{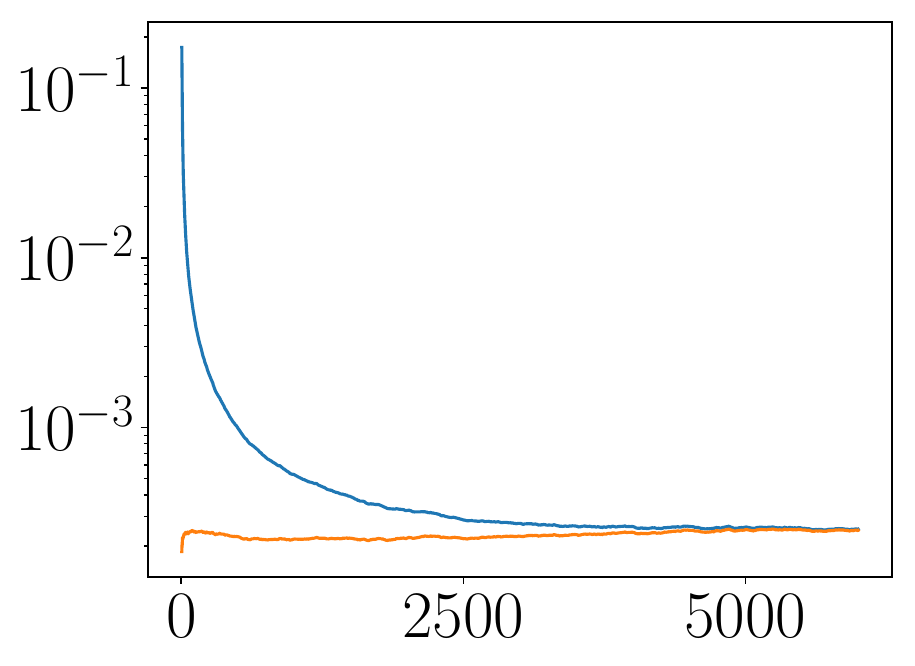} & 
 \includegraphics[width=0.3\textwidth]{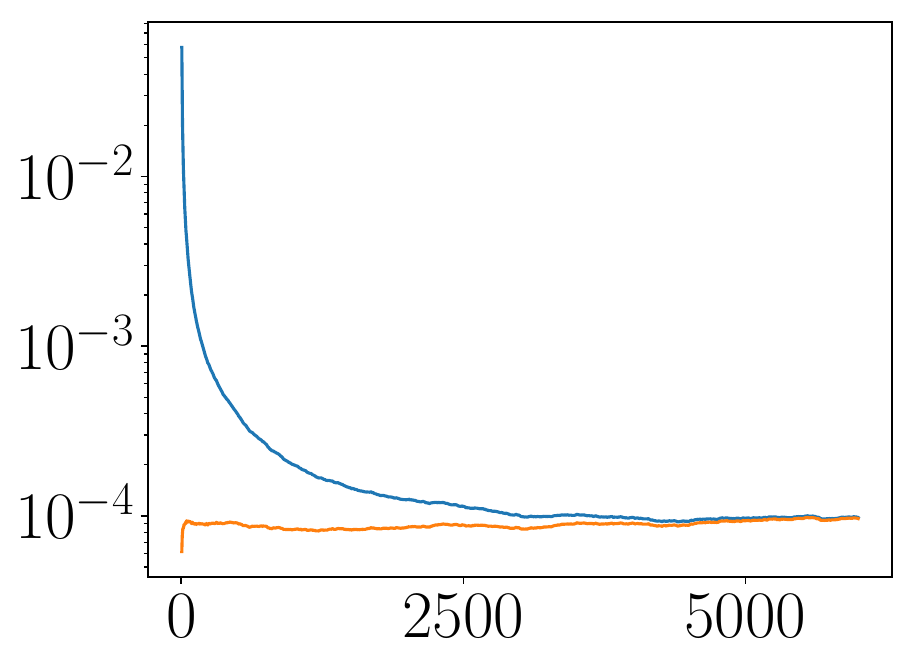}& 
 \includegraphics[width=0.3\textwidth]{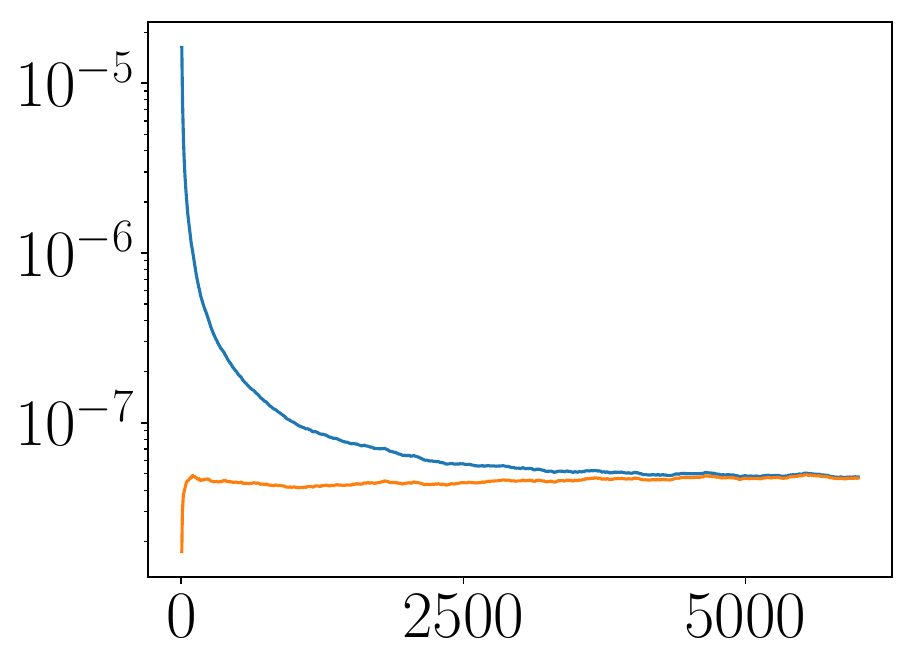} \\
 $\mu=0.1$ & $\mu=1$ & $\mu=10$ \\
\end{tabular}
\vspace{-0.2cm}
% \caption{The simulation of $\Var[\frac{1}{v_t}]$ and $\Var[\frac{c_t}{v_t}]$. The x-axis is iteration number (the simulation starts from 5) and the y-axis is the variance in the log scale.}
\caption{The simulation of $\Var[\frac{1}{v_t}]$ and $\Var[\frac{c_t}{v_t}]$. The x-axis is iteration \# (from 5), the y-axis is the variance (log scale).}
    \label{fig:vt_var_simulation}
\end{figure}
\end{minipage}
\vspace{-0.4cm}
\end{figure}

% We would like to thank Prof. Jiawei Han for his generous help. 

\bibliography{iclr2020_conference}

\begin{thebibliography}{34}
\providecommand{\natexlab}[1]{#1}
\providecommand{\url}[1]{\texttt{#1}}
\expandafter\ifx\csname urlstyle\endcsname\relax
  \providecommand{\doi}[1]{doi: #1}\else
  \providecommand{\doi}{doi: \begingroup \urlstyle{rm}\Url}\fi

\bibitem[Ba et~al.(2016)Ba, Kiros, and Hinton]{ba2016layer}
Jimmy~Lei Ba, Jamie~Ryan Kiros, and Geoffrey~E Hinton.
\newblock Layer normalization.
\newblock \emph{arXiv preprint arXiv:1607.06450}, 2016.

\bibitem[Balduzzi et~al.(2017)Balduzzi, Frean, Leary, Lewis, Ma, and
  McWilliams]{balduzzi2017shattered}
David Balduzzi, Marcus Frean, Lennox Leary, JP~Lewis, Kurt Wan-Duo Ma, and
  Brian McWilliams.
\newblock The shattered gradients problem: If resnets are the answer, then what
  is the question?
\newblock In \emph{ICML}, 2017.

\bibitem[Bengio et~al.(2013)Bengio, Boulanger-Lewandowski, and
  Pascanu]{bengio2013advances}
Yoshua Bengio, Nicolas Boulanger-Lewandowski, and Razvan Pascanu.
\newblock Advances in optimizing recurrent networks.
\newblock In \emph{2013 IEEE International Conference on Acoustics, Speech and
  Signal Processing}, pp.\  8624--8628. IEEE, 2013.

\bibitem[Bernstein et~al.(2018)Bernstein, Wang, Azizzadenesheli, and
  Anandkumar]{bernstein2018signsgd}
Jeremy Bernstein, Yu-Xiang Wang, Kamyar Azizzadenesheli, and Anima Anandkumar.
\newblock signsgd: Compressed optimisation for non-convex problems.
\newblock In \emph{ICML}, 2018.

\bibitem[Cauchy(1847)]{cauchy1847methode}
Augustin Cauchy.
\newblock M{\'e}thode g{\'e}n{\'e}rale pour la r{\'e}solution des systemes
  d’{\'e}quations simultan{\'e}es.
\newblock \emph{Comp. Rend. Sci. Paris}, 25\penalty0 (1847):\penalty0 536--538,
  1847.

\bibitem[Cettolo et~al.(2014)Cettolo, Niehues, St{\"u}ker, Bentivogli, and
  Federico]{cettolo2014report}
Mauro Cettolo, Jan Niehues, Sebastian St{\"u}ker, Luisa Bentivogli, and
  Marcello Federico.
\newblock Report on the 11th iwslt evaluation campaign, iwslt 2014.
\newblock In \emph{Proceedings of the International Workshop on Spoken Language
  Translation,}, 2014.

\bibitem[Chelba et~al.(2013)Chelba, Mikolov, Schuster, Ge, Brants, Koehn, and
  Robinson]{Chelba2013OneBW}
Ciprian Chelba, Tomas Mikolov, Michael Schuster, Qi~Ge, Thorsten Brants,
  Phillipp Koehn, and Tony Robinson.
\newblock One billion word benchmark for measuring progress in statistical
  language modeling.
\newblock In \emph{INTERSPEECH}, 2013.

\bibitem[Chen et~al.(2018)Chen, Zhou, Tang, Yang, and Gu]{chen2018closing}
Jinghui Chen, Dongruo Zhou, Yiqi Tang, Ziyan Yang, and Quanquan Gu.
\newblock Closing the generalization gap of adaptive gradient methods in
  training deep neural networks.
\newblock \emph{arXiv preprint arXiv:1806.06763}, 2018.

\bibitem[Deng et~al.(2009)Deng, Dong, Socher, Li, Li, and
  Fei-Fei]{deng2009imagenet}
Jia Deng, Wei Dong, Richard Socher, Li-Jia Li, Kai Li, and Li~Fei-Fei.
\newblock Imagenet: A large-scale hierarchical image database.
\newblock In \emph{ICML}, 2009.

\bibitem[Devlin et~al.(2019)Devlin, Chang, Lee, and Toutanova]{devlin2018bert}
Jacob Devlin, Ming-Wei Chang, Kenton Lee, and Kristina Toutanova.
\newblock Bert: Pre-training of deep bidirectional transformers for language
  understanding.
\newblock In \emph{NAACL-HLT}, 2019.

\bibitem[Dozat(2016)]{dozat2016incorporating}
Timothy Dozat.
\newblock Incorporating nesterov momentum into adam.
\newblock 2016.

\bibitem[Duchi et~al.(2010)Duchi, Hazan, and Singer]{duchi2011adaptive}
John Duchi, Elad Hazan, and Yoram Singer.
\newblock Adaptive subgradient methods for online learning and stochastic
  optimization.
\newblock In \emph{COLT}, 2010.

\bibitem[Gauss(1823)]{gauss1823theoria}
Carl-Friedrich Gauss.
\newblock Theoria combinationis observationum erroribus minimis obnoxiae.
\newblock \emph{Commentationes Societatis Regiae Scientiarum Gottingensis
  Recentiores}, 1823.

\bibitem[Gotmare et~al.(2019)Gotmare, Keskar, Xiong, and Socher]{gotmare2018a}
Akhilesh Gotmare, Nitish~Shirish Keskar, Caiming Xiong, and Richard Socher.
\newblock A closer look at deep learning heuristics: Learning rate restarts,
  warmup and distillation.
\newblock In \emph{ICLR}, 2019.

\bibitem[Goyal et~al.(2017)Goyal, Doll{\'a}r, Girshick, Noordhuis, Wesolowski,
  Kyrola, Tulloch, Jia, and He]{goyal2017accurate}
Priya Goyal, Piotr Doll{\'a}r, Ross Girshick, Pieter Noordhuis, Lukasz
  Wesolowski, Aapo Kyrola, Andrew Tulloch, Yangqing Jia, and Kaiming He.
\newblock Accurate, large minibatch sgd: Training imagenet in 1 hour.
\newblock \emph{arXiv preprint arXiv:1706.02677}, 2017.

\bibitem[He et~al.(2016)He, Zhang, Ren, and Sun]{he2016deep}
Kaiming He, Xiangyu Zhang, Shaoqing Ren, and Jian Sun.
\newblock Deep residual learning for image recognition.
\newblock In \emph{CVPR}, 2016.

\bibitem[Hinton et~al.(2012)Hinton, Srivastava, and
  Swersky]{tieleman2012lecture}
Geoffrey Hinton, Nitish Srivastava, and Kevin Swersky.
\newblock Neural networks for machine learning lecture 6a overview of
  mini-batch gradient descent.
\newblock \emph{Cited on}, 2012.

\bibitem[Ioffe \& Szegedy(2015)Ioffe and Szegedy]{ioffe2015batch}
Sergey Ioffe and Christian Szegedy.
\newblock Batch normalization: Accelerating deep network training by reducing
  internal covariate shift.
\newblock In \emph{ICML}, 2015.

\bibitem[Kingma \& Ba(2014)Kingma and Ba]{kingma2014adam}
Diederik~P Kingma and Jimmy Ba.
\newblock Adam: A method for stochastic optimization.
\newblock In \emph{ICLR}, 2014.

\bibitem[Krizhevsky et~al.(2009)Krizhevsky, Hinton,
  et~al.]{krizhevsky2009learning}
Alex Krizhevsky, Geoffrey Hinton, et~al.
\newblock Learning multiple layers of features from tiny images.
\newblock Technical report, Citeseer, 2009.

\bibitem[Liu et~al.(2018)Liu, Ren, Shang, Peng, and Han]{liu2018efficient}
Liyuan Liu, Xiang Ren, Jingbo Shang, Jian Peng, and Jiawei Han.
\newblock Efficient contextualized representation: Language model pruning for
  sequence labeling.
\newblock \emph{EMNLP}, 2018.

\bibitem[Liu et~al.(2019)Liu, Ott, Goyal, Du, Joshi, Chen, Levy, Lewis,
  Zettlemoyer, and Stoyanov]{liu2019roberta}
Yinhan Liu, Myle Ott, Naman Goyal, Jingfei Du, Mandar Joshi, Danqi Chen, Omer
  Levy, Mike Lewis, Luke Zettlemoyer, and Veselin Stoyanov.
\newblock Roberta: A robustly optimized bert pretraining approach.
\newblock \emph{arXiv preprint arXiv:1907.11692}, 2019.

\bibitem[Loshchilov \& Hutter(2018)Loshchilov and Hutter]{loshchilov2017fixing}
Ilya Loshchilov and Frank Hutter.
\newblock Fixing weight decay regularization in adam.
\newblock In \emph{ICLR}, 2018.

\bibitem[Luo et~al.(2019)Luo, Xiong, Liu, and Sun]{luo2019adaptive}
Liangchen Luo, Yuanhao Xiong, Yan Liu, and Xu~Sun.
\newblock Adaptive gradient methods with dynamic bound of learning rate.
\newblock In \emph{ICLR}, 2019.

\bibitem[Nau(2014)]{nau2014forecasting}
Robert Nau.
\newblock Forecasting with moving averages.
\newblock 2014.

\bibitem[Ott et~al.(2019)Ott, Edunov, Baevski, Fan, Gross, Ng, Grangier, and
  Auli]{ott2019fairseq}
Myle Ott, Sergey Edunov, Alexei Baevski, Angela Fan, Sam Gross, Nathan Ng,
  David Grangier, and Michael Auli.
\newblock fairseq: A fast, extensible toolkit for sequence modeling.
\newblock In \emph{NAACL}, 2019.

\bibitem[Popel \& Bojar(2018)Popel and Bojar]{popel2018training}
Martin Popel and Ond{\v{r}}ej Bojar.
\newblock Training tips for the transformer model.
\newblock \emph{The Prague Bulletin of Mathematical Linguistics}, 110\penalty0
  (1):\penalty0 43--70, 2018.

\bibitem[Reddi et~al.(2018)Reddi, Kale, and Kumar]{reddi2019convergence}
Sashank~J Reddi, Satyen Kale, and Sanjiv Kumar.
\newblock On the convergence of adam and beyond.
\newblock In \emph{ICLR}, 2018.

\bibitem[Szegedy et~al.(2016)Szegedy, Vanhoucke, Ioffe, Shlens, and
  Wojna]{szegedy2016rethinking}
Christian Szegedy, Vincent Vanhoucke, Sergey Ioffe, Jon Shlens, and Zbigniew
  Wojna.
\newblock Rethinking the inception architecture for computer vision.
\newblock In \emph{CVPR}, 2016.

\bibitem[Vaswani et~al.(2017)Vaswani, Shazeer, Parmar, Uszkoreit, Jones, Gomez,
  Kaiser, and Polosukhin]{vaswani2017attention}
Ashish Vaswani, Noam Shazeer, Niki Parmar, Jakob Uszkoreit, Llion Jones,
  Aidan~N Gomez, {\L}ukasz Kaiser, and Illia Polosukhin.
\newblock Attention is all you need.
\newblock In \emph{NIPS}, 2017.

\bibitem[Wolter(2007)]{wolter2007taylor}
Kirk~M Wolter.
\newblock Taylor series methods.
\newblock In \emph{Introduction to variance estimation}. 2007.

\bibitem[Xiao et~al.(2017)Xiao, Yu, Lin, and Chen]{Xiao2017DSCOVRRP}
Lin Xiao, Adams~Wei Yu, Qihang Lin, and Weizhu Chen.
\newblock Dscovr: Randomized primal-dual block coordinate algorithms for
  asynchronous distributed optimization.
\newblock \emph{J. Mach. Learn. Res.}, 2017.

\bibitem[Zeiler(2012)]{zeiler2012adadelta}
Matthew~D Zeiler.
\newblock Adadelta: an adaptive learning rate method.
\newblock \emph{arXiv preprint arXiv:1212.5701}, 2012.

\bibitem[Zhang et~al.(2019)Zhang, Dauphin, and Ma]{zhang2019fixup}
Hongyi Zhang, Yann~N Dauphin, and Tengyu Ma.
\newblock Fixup initialization: Residual learning without normalization.
\newblock In \emph{ICLR}, 2019.

\end{thebibliography}
\bibliographystyle{iclr2020_conference}

\newpage
\appendix
% !TEX encoding = UTF-8
% !TEX Root = 0_main.tex

\section{Proof of Theorem~\ref{theorem: variance_mono}}\label{app:proof_mono}

For ease of notation, we refer $\psi^2(.)$ as $x$ and $\frac{1}{\sigma^2}$ as $\tau^2$. 
Thus, $x \sim \sics(\rho, \tau^2)$ and:
\begin{align}
\vspace{-0.3in}
p(x) = \frac{(\tau^2\rho/2)^{\rho/2}}{\Gamma(\rho/2)}\frac{\exp[\frac{-\rho\tau^2}{2x}]}{x^{1 + \rho/2}}
\quad\mbox{and}\quad
\E[x] = \frac{\rho}{(\rho - 2) \sigma^2} \;(\forall\, \rho > 2)
\label{eqn:inv-chi-sq-pdf-mean}
\vspace{-0.3in}
\end{align}
where $\Gamma(.)$ is the gamma function. Therefore, we have:
\begin{align}
\vspace{-0.3in}
    \E[\sqrt{x}] = \int_{0}^{\infty} \sqrt{x}\, p(x)\,dx = \frac{\tau \sqrt{\rho} \,\Gamma( (\rho - 1)/2)}{\sqrt{2} \,\Gamma(\rho/2)}\; (\forall\, \rho > 4).
    \label{eqn:expect-sqrt-x}
\vspace{-0.3in}
\end{align}
Based on Equation~\ref{eqn:inv-chi-sq-pdf-mean} and \ref{eqn:expect-sqrt-x}, for $\forall\, \rho > 4$, we have:
\begin{align}
\vspace{-0.5in}
\Var[\psi(.)] = \Var[\sqrt{x}] = \E[x] - \E[\sqrt{x}]^2 = \tau^2 (\frac{\rho}{\rho-2} - \frac{\rho \,2^{2\rho - 5}}{\pi}\gB(\frac{\rho-1}{2}, \frac{\rho-1}{2})^2),
\vspace{-0.5in}
\end{align}
where $\gB(.)$ is the beta function. To prove the monotonic property of $\Var[\psi(.)]$, we need to show:

\begin{lemma}
for $t \geq 4$, $\frac{\partial }{\partial t} (\frac{t}{t-2} - \frac{t \,2^{2t - 5}}{\pi}\gB(\frac{t-1}{2}, \frac{t-1}{2})^2) < 0$
\end{lemma}

\begin{proof}
The target inequality can be re-wrote as 
\begin{align*}
    &  \frac{\partial }{\partial t} (\frac{t}{t-2} - \frac{t \,2^{2t - 5}}{\pi}\gB(\frac{t-1}{2}, \frac{t-1}{2})^2) \\
    &= \frac{-2}{(t-2)^2} 
    - \frac{2^{2t - 5}}{\pi}\gB(\frac{t-1}{2},
    \frac{t-1}{2})^2
    - \frac{t \,2^{2t - 5} \ln 4}{\pi}\gB(\frac{t-1}{2}, \frac{t-1}{2})^2
    \\
    & - \frac{2t \,2^{2t - 5}}{\pi}\gB(\frac{t-1}{2}, \frac{t-1}{2})^2 (\Psi(\frac{t-1}{2})-\Psi(t-1)), \quad \left( \Psi(x) = \frac{\Gamma'(x)}{\Gamma(x)} \right)
    \\ & < 0
    \\
\end{align*}
This inequality is equivalent to:
\begin{align*}
    \frac{64 \pi }{(t-2)^2 4^t \cB(\frac{t-1}{2},\frac{t-1}{2})^2} + 1 + t \ln 4+2t\Psi(\frac{t-1}{2}) \\
    > 2t\Psi(t-1) \stackrel{(i)}{=} t[\Psi(\frac{t-1}{2}) + \Psi(\frac{t}{2})+ \ln 4],
\end{align*}
where $(i)$ is derived from Legendre duplication formula.
Simplify the above inequality, we get:
\begin{align*}
    \frac{64 \pi }{(t-2)^2 4^t \cB(\frac{t-1}{2},\frac{t-1}{2})^2} + 1 + t \Psi(\frac{t-1}{2}) - t\Psi(\frac{t}{2}) > 0,
\end{align*}

We only need to show 
\begin{align*}
    &\frac{64 \pi }{(t-2)^2 4^t \cB(\frac{t-1}{2},\frac{t-1}{2})^2} + 1 + t \Psi(\frac{t-1}{2}) - t\Psi(\frac{t}{2})
    \\
    &\geq \frac{64 \pi }{(t-2)^2 4^t \cB(\frac{t-1}{2},\frac{t-1}{2})^2} + 2 + t ( \ln(t/2)-1/(t/2-0.5) ) - t\ln (t/2)
    \\
    &= \frac{64 \pi }{(t-2)^2 4^t \cB(\frac{t-1}{2},\frac{t-1}{2})^2} - \frac{2}{t-1}
    \\
    &> \frac{64 \pi }{(t-2)^2 4^t \cB(\frac{t-1}{2},\frac{t-1}{2})^2} - \frac{2}{t-2} \geq 0 ,
\end{align*}
where the first inequality is from $\ln(x)-1/(2x)>\Psi(x) > \ln(x+0.5)-1/x$.

Therefore, we only need to show 
\begin{align*}
32 \pi \geq (t-2) 4^t \cB(\frac{t-1}{2},\frac{t-1}{2})^2,
\end{align*}

which is equivalent to  
\begin{align*}
&(t-2) 4^t \cB(\frac{t-1}{2},\frac{t-1}{2})^2
= (t-2) 4^t \frac{\Gamma(\frac{t-1}{2})^4}{\Gamma(t-1)^2}
% = (t-2) 4^t \frac{\Gamma(\frac{t-1}{2})^4}{\Gamma(t-1)^2} 
\\ &\stackrel{(i)}{=} (t-2) 4^t \frac{\Gamma(\frac{t-1}{2})^2}{\Gamma(t/2)^2} 4^{2-t} \pi = 16 \pi (t-2)  \frac{\Gamma(\frac{t-1}{2})^2}{\Gamma(t/2)^2} \leq 32 \pi,
\end{align*}
where $(i)$ is from Legendre duplication formula.

So we only need to show 
\begin{align}
(t-2)  \frac{\Gamma(\frac{t-1}{2})^2}{\Gamma(t/2)^2} \leq 2
\end{align}

Using Gautschi's inequality ($\frac{\Gamma(x+1)}{\Gamma(x+s)} < (x+1)^{1-s}$), we have 
\begin{align}
(t-2)  \frac{\Gamma(\frac{t-1}{2})^2}{\Gamma(t/2)^2}  \leq (t-2) (\frac{t-1}{2})^{-1} = \frac{2(t-2)}{t-1} < 2
\end{align}
\end{proof}

\section{Implementation Details} \label{app:implement}
\subsection{Language Modeling}
Our implementation is based on the previous work~\citep{liu2018efficient}.
Specifically, we use two-layer LSTMs with 2048 hidden states with adaptive softmax to conduct experiments on the one billion words dataset.
Word embedding (random initialized) of 300 dimensions is used as the input and the adaptive softmax is incorporated with a default setting (cut-offs are set to $[4000,40000,200000]$). 
Additionally, as pre-processing, we replace all tokens occurring equal or less than 3 times with as UNK. 
Dropout is applied to each layer with a ratio of $0.1$, gradients are clipped at 5.0.
We use the default hyper-parameters to update moving averages, \ie $\beta_1=0.9$ and $\beta_2=0.999$.
The learning rate is set to start from 0.001, and decayed at the start of 10th epochs. 
LSTMs are unrolled for 20 steps without resetting the LSTM states and the batch size is set to 128.
All models are trained on one NVIDIA Tesla V100 GPU. 

\subsection{Imageine Classification}
We use the default ResNet architectures~\citep{he2016deep} in a public pytorch re-implementation\footnote{\url{https://github.com/bearpaw/pytorch-classification}}.
Specifically, we use $20$-layer ResNet ($9$ Basic Blocks) for CIFAR-10 and 18-layer ResNet ($8$ Basic Blocks) for ImageNet.
Batch size is $128$ for CIFAR-10 and $256$ for ImageNet. The model is trained for $186$ epoches and the learning rate decays at the $81$-th and the $122$-th epoches by $0.1$ on CIFAR-10, while the model is trained for $90$ epoches and the learning rate decays at the $31$-th and the $61$-th epoch by $0.1$ on ImageNet.  
For Adam and RAdam, we set $\beta_1=0.9, \beta_2=0.999$. For SGD, we set the momentum factor as $0.9$. The weight decay rate is $10^{-4}$. Random cropping and random horizontal flipping are applied to training data.

\subsection{Neural Machine Translation}
Our experiments are based on the default Transformers~\citep{vaswani2017attention} implementation from the fairseq package~\citep{ott2019fairseq}. 
Specifically, we use word embedding with 512 dimensions and 6-layer encoder / decoder with 4 head and 1024 feedforward dimensions on the IWSLT14' dataset; use word embedding with 512 dimension and 6-layer encoder/decoder with 8 heads and 2048 feedforward dimensions on the WMT14' dataset. 
Label smoothed cross entropy is used as the objective function with an uncertainty $= 0.1$~\citep{szegedy2016rethinking}. 
We use linear learning rate decay starting from $3e^{-4}$, and the checkpoints of the last $20$ epoches are averaged before evaluation. 
As to the wamrup strategy, we use a linear warmup for Adam in the first $4000$ updates, and set $\beta_2$ to satisfy $\nu=4000$ ($\beta_2 = 0.9995$).
In the IWSLT'14 dataset, we conduct training on one NVIDIA Tesla V100 GPU, set maximum batch size as $4000$, apply dropout with a ratio $0.3$, using weight decay of $0.0001$ and clip the gradient norm at $25$. 
In the WMT'16 dataset, we conduct training on four NVIDIA Quadro R8000 GPUs and set maximum batch size as $8196$.

\section{Downgrading to SGDM} \label{app:foursteps}

As a byproduct determined by math derivations, we degenerated RAdam to SGD with momentum in the first several updates. 
Although this stage only contains several gradient updates, these updates could be quite damaging (e.g., in our Figure~\ref{fig:histogram_2}, the gradient distribution is distorted within 10 gradient updates). 
Intuitively, updates with divergent adaptive learning rate variance could be more damaging than the ones with converged variance, as divergent variance implies more instability. 
As a case study, we performed experiments on the CIFAR10 dataset. 
Five-run average results are summarized in Table~\ref{tab:foursteps}. 
The optimizer fails to get an equally reliably model when changing the first 4 updates to Adam, yet the influence of switching is less deleterious when we change 5-8 updates instead. 
This result verifies our intuition and is in agreement with our theory — the first few updates could be more damaging than later updates. 
By saying that, we still want to emphasize that this part (downgrading to SGDM) is only a minor part of our algorithm design whereas our main focus is on the mechanism of warmup and the derivation of the rectification term.

\begin{table}[h]
\caption{Performance on CIFAR10 (lr = 0.1).}
\begin{center}
\begin{tabularx}{\columnwidth}{l l l *{3}{Y}}
% \hline
\toprule
1-4 steps & 5-8 steps & 8+ steps & test acc & train loss & train error\\
\midrule
% \multicolumn{3}{c}{RAdam}
RAdam & RAdam & RAdam & 91.08 & 0.021 & 0.74 \\
\midrule
Adam (w. divergent var.) & RAdam & RAdam & 89.98 & 0.060 & 2.12 \\
\midrule
SGD & Adam (w. convergent var.) & RAdam & 90.29 & 0.038 & 1.23\\
\bottomrule
\end{tabularx}
\end{center}
\label{tab:foursteps}
\end{table}

\end{document}